\documentclass[twoside]{article}

\usepackage[accepted]{aistats2026}
\usepackage[round]{natbib}

\runningtitle{Variable Importance with Unobserved Confounding and the Rashomon Effect}

\usepackage[utf8]{inputenc} 
\usepackage[T1]{fontenc}    
\usepackage{hyperref}       
\usepackage{url}            
\usepackage{booktabs}       
\usepackage{amsfonts}       
\usepackage{nicefrac}       
\usepackage{microtype}      
\usepackage{xcolor}         
\usepackage{amsmath}
\usepackage{amsthm}
\usepackage{amssymb}
\usepackage{graphicx}

\usepackage{tabularx}

\newtheorem{definition}{Definition}
\newtheorem{theorem}{Theorem}
\newtheorem{lemma}{Lemma}
\newtheorem{proposition}{Proposition}
\newtheorem{assumption}{Assumption}
\newtheorem{corollary}{Corollary}

\newcommand{\Obs}{X}
\newcommand{\Unobs}{U}
\newcommand{\obs}{x}
\newcommand{\unobs}{u}
\newcommand{\D}{\mathcal{D}^{(n)}}
\newcommand{\ObsDist}{\mathcal{P}_{XY}}

\newcommand{\CondDist}[1]{\mathcal{P}_{XY|U=#1}}
\newcommand{\UnobsDist}{\mathcal{P}_{U}}

\newcommand{\Xspace}{\mathcal{X}}
\newcommand{\Yspace}{\mathcal{Y}}
\newcommand{\Uspace}{\mathcal{U}}

\newcommand{\optmodel}{g^*}
\newcommand{\condsubmodel}[2]{f_{#1}(#2)}
\newcommand{\condsubmodelNoInput}[1]{f_{#1}}
\newcommand{\maxcondloss}{\epsilon_{unobs}}
\newcommand{\maxdistshift}{\tau}
\newcommand{\allcondmodels}{S^*}
\newcommand{\poprset}{\mathcal{R}}
\newcommand{\estrset}{\hat{\mathcal{R}}^{(n)}}

\newcommand{\smarteps}{\varepsilon_n}
\newcommand{\estsmarteps}{\hat{\varepsilon}_n}

\newcommand{\suplambda}{\lambda_{\sup}}

\newcommand{\vibound}{\alpha}

\newcommand{\ourmethodAcronym}{UNIVERSE}
\newcommand{\ourmethodName}{UNobservables and Inference for Variable importancE using Rashomon SEts}

\definecolor{gold}{RGB}{255, 215, 0}
\newcommand{\goldlozenge}{{\color{gold} \blacklozenge}}

\begin{document}

\twocolumn[
\aistatstitle{Doctor Rashomon and the UNIVERSE of Madness: Variable Importance with Unobserved Confounding and the Rashomon Effect}

\aistatsauthor{ Jon Donnelly$^*$ \And Srikar Katta$^*$ \And  Emanuele Borgonovo \And Cynthia Rudin }

\aistatsaddress{ Duke University \And  Duke University \And Bocconi University \And Duke University } ]

\begin{abstract}

    Variable importance (VI) methods are often used for hypothesis generation, feature selection, and scientific validation. In the standard VI pipeline, an analyst estimates VI for a \textit{single} predictive model with only the \textit{observed} features. However, the importance of a feature depends heavily on which other variables are included in the model, and essential variables are often omitted from observational datasets. Moreover, the VI estimated for one model is often not the same as the VI estimated for another equally-good model -- a phenomenon known as the \textit{Rashomon Effect}. We address these gaps by introducing \ourmethodName~(\ourmethodAcronym). Our approach adapts Rashomon sets -- the sets of near-optimal models in a dataset -- to produce bounds on the true VI even with missing features. We theoretically guarantee the robustness of our approach, show strong performance on semi-synthetic simulations, and demonstrate its utility in a credit risk task. 
\end{abstract}

\section{INTRODUCTION} \label{sec:intro}
Variable importance (VI) methods are used throughout science to study the strength of different risk factors, generate new hypotheses, and understand which features are worth collecting for future predictions. In the standard variable importance pipeline, a scientist calculates the importance of a variable for a \textit{single} predictive model using only the \textit{observed} features. However, this approach may lead to misleading insights because multiple models may explain a dataset equally well -- a phenomenon known as the \textit{Rashomon effect} \citep{breiman2001statistical} -- and a variable deemed important for one model may not be important for another. To overcome the Rashomon effect, recent research has developed algorithms to compute the importance of variables across the \textit{Rashomon set}, the set of all near-optimal models for a given dataset \citep{fisher2019all, dong2020exploring, donnelly2023rashomon, xin2022exploring, chen2023understanding, babbar2025near}. These methods are able to discover the true importance of variables in experiments \cite{donnelly2023rashomon} but are not valid when there may be important unobserved variables, which is often the case in practice. 
\noindent\let\thefootnote\relax\footnotetext{\hspace{-2em}$^*$Jon Donnelly and Srikar Katta contributed equally to this work}


\begin{figure}[t]
    \centering
    \includegraphics[width=1\linewidth]{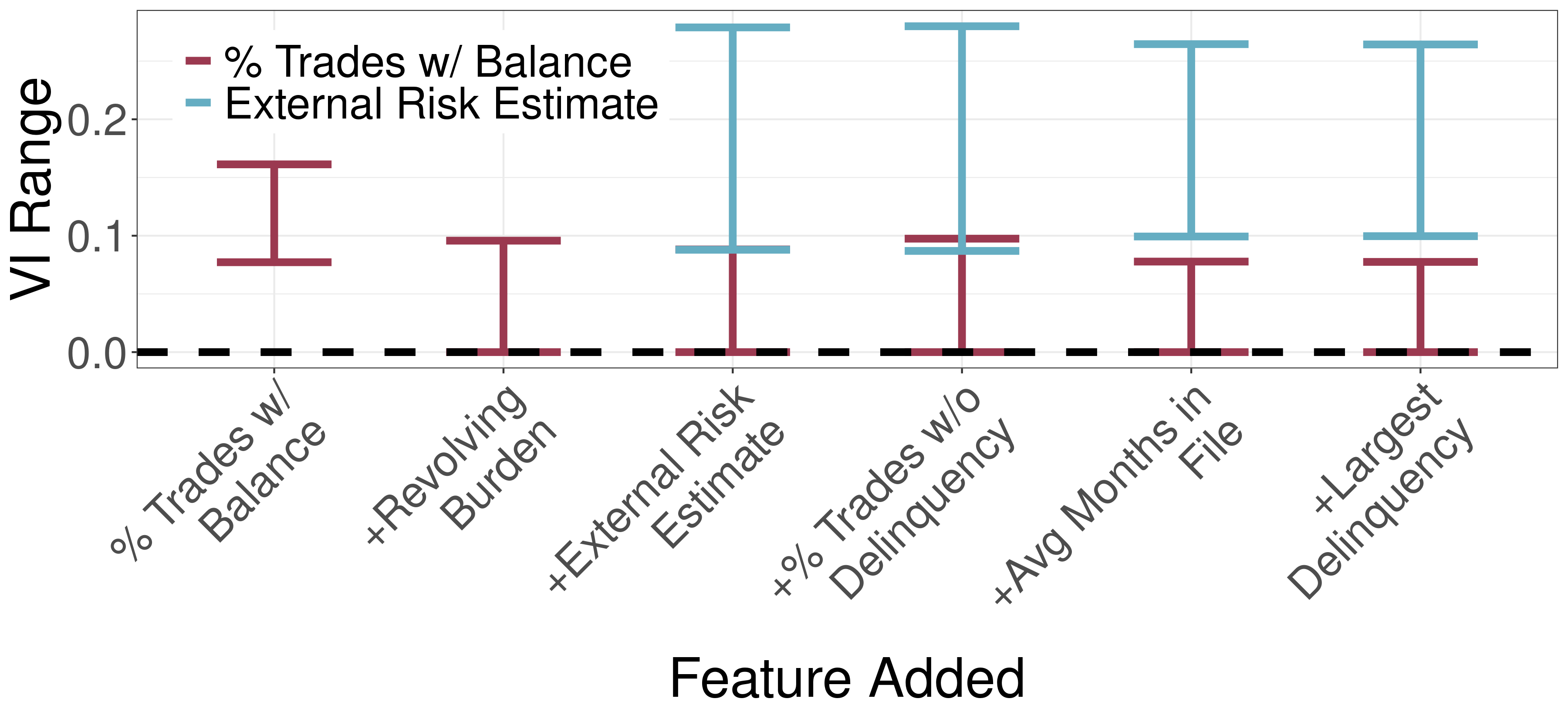}
    \caption{\textbf{Variable importance intervals for two variables from the FICO dataset, as more variables are observed.} When we estimate the range of variable importance  \citep[quantified by subtractive model reliance of][]{fisher2019all} with only one observed feature -- \texttt{\% Trades w/ Balance} -- this feature seems important, with a confidence interval that does not overlap with 0. However, as we introduce more variables into the analysis pipeline, this variable's confidence intervals overlap with 0, no longer yielding significant results. In contrast, \texttt{External Risk Estimate} -- the third feature incorporated into the analysis -- remains significant even after controlling for other observed variables. Thus, observing additional information changes the conclusion a practitioner would draw about \texttt{\% Trades w/ Balance}.}
    \label{fig:real_intervals_intro}
\end{figure}

Figure~\ref{fig:real_intervals_intro} illustrates this problem: we present the range of variable importance for two variables across the Rashomon set as we incorporate more features into the model to predict loan default risk. 
When only a single feature is observed, the \texttt{\%Trades w/ Balance} feature appears statistically significant for predicting loan default. But as we incorporate features that we know are important, this variable no longer seems important. This inconsistency shows how omitted variables can undermine decision-making.

In this work, we introduce \ourmethodName~(\ourmethodAcronym) to address both challenges. We take advantage of Rashomon sets to derive bounds on the true variable importance under the possibility of unobserved variables and adapt our Rashomon sets to account for finite sample uncertainty. We prove theoretically that our bounds contain the true variable importance for the true conditional mean function under unobserved confounding and have valid Type-1 error control. Via semi-synthetic experiments, we demonstrate that our bounds are tight enough to draw useful conclusions in practice, even in the presence of unmeasured variables and the Rashomon effect. The code for this work is available at \href{https://github.com/jdonnelly36/UNIVERSE}{https://github.com/jdonnelly36/UNIVERSE}.

\section{RELATED WORK} \label{sec:related_work}

We review the literature on (i) variable importance measures, (ii) the Rashomon effect, and (iii) unobserved confounding. To our knowledge, our work is the first to consider all three components simultaneously.

\paragraph{Variable Importance} 
The most classical examples of variable importance metrics include the parameters of a linear model and the permutation importance from a decision tree \citep{louppe2013understanding, kazemitabar2017variable}. Because these metrics may be restrictive, researchers have also introduced perturbation-based methods that can be applied to any model, which quantify how much a specified model's performance changes as variables are perturbed. Such methods include Shapley additive explanations (SHAP) \citep{lundberg2017unified}, local interpretable model-agnostic explanations (LIME) \citep{ribeiro2016should}, model reliance \citep{breiman2001random, fisher2019all}, and conditional model reliance \citep{fisher2019all}. However, these approaches only estimate the importance of a feature for the \textit{specified} model and do not consider the \textit{Rashomon Effect}.

\paragraph{The Rashomon Effect} The Rashomon Effect -- coined by \citet{breiman2001statistical} -- describes the phenomenon in which multiple, possibly distinct, models describe a given dataset equally well. 
The Rashomon effect is ubiquitous \citep{paes2023inevitability} and has been observed in high-stakes domains, making the development of tools that are robust to the Rashomon effect essential for trustworthy decision making \citep{rudin2024amazing}. 
The most common approach to ensure robustness to the Rashomon effect is to estimate a Rashomon \textit{ set} -- the set of all near-optimal models in a model class. Recent advancements have led to the estimation/approximation of Rashomon sets for decision trees \citep{xin2022exploring, babbar2025near}, risk scores \citep{liu2022fasterrisk}, generalized additive models \citep{chen2023understanding}, and prototypical part neural networks \citep{donnelly2025rashomon}. Analysts can then find the range, point cloud, or distribution of variable importance across Rashomon sets \citep{fisher2019all, dong2020exploring, donnelly2023rashomon}, or identify variables are consistently more important than others across the Rashomon set \cite{laberge2023partial}.

Several approaches also offer a related type of robustness -- to model misspecification stemming from finite sample model uncertainty. Some methods are built for specific variable importance metrics, such as Leave One Covariate Out (LOCO) or SHapley Additive Explanations  \citep{williamson2021general, williamson2020efficient, zhang2020floodgate, massimo2022floodgate, verdinelli2024feature}. While other methods should guarantee robustness due to finite-sample model misspecification \citep{lei2018distribution}, these approaches are not robust to unobserved confounding.


\paragraph{Unobserved confounding in causal inference} Unobserved features pose a pervasive problem in causal inference: when some important confounders are unobserved, treatment effects estimated on the observed data will be biased. To overcome this issue, analysts often design sensitivity analyses to identify a \textit{set} of viable causal effects under a reasonable level of possible confounding \citep{rosenbaum1987sensitivity, rosenbaum2007sensitivity, manski2003partial}. We take a similar approach: we identify upper and lower bounds on variable importance under a reasonable level of unobserved confounding. Unlike causal sensitivity analyses that require a new approach for each new causal estimand, our approach is generic and can handle \textit{any} model-based variable importance measure. While some approaches have relatively mild assumptions for causal estimands, these are too restrictive for variable importance metrics. For example, the most flexible sensitivity analysis framework by \citet{chernozhukov2022long} still requires that the causal estimand is a linear functional of the data but common variable importance metrics like SHAP are \textit{quadratic} functionals \citep{verdinelli2024feature}. Our approach overcomes these restrictive assumptions and is useful for a wide class of VI metrics.

\section{METHODS} \label{sec:methods}

Let $\D = \{(\Obs_i, Y_i)\}_{i=1}^n$ represent a dataset consisting of $n$ independent and identically distributed (i.i.d.) tuples drawn from some distribution $\ObsDist$, where $Y_i \in \Yspace \subseteq \mathbb{R}$ is the prediction target of interest and $X_i \in \Xspace \subseteq \mathbb{R}^p$ is a vector of $p$ covariates.

Consider a function class $\mathcal{F}$ consisting of models that output a prediction given the observed variables $\Obs_i$ (e.g., the space of all possible sparse decision trees constructed using $\Obs_i$). 
Let $\phi_j$ denote a function that quantifies the importance of variable $j$ to some model $f \in \mathcal{F}$ for some observation $(\Obs_i, Y_i)$; this is a \textit{local} variable importance quantity, such as local SHAP. We describe how some common variable importance metrics, including permutation importance and SHAP, can be expressed in these terms in Appendix~\ref{app:applicable_vi_metrics}. In our experiments, we focus on subtractive model reliance \citep{fisher2019all} because it is simple to interpret and easy to compute. Let $\Phi_j(f, \ObsDist) := \mathbb{E}_{(\Obs, Y) \sim \ObsDist}[\phi_j(f, (\Obs, Y))]$ represent the \textit{average} importance of variable $j$ over the population data distribution $\ObsDist$; for example, when our target quantity is global SHAP for feature $j$ and the model $f$, $\phi_j(f, (\Obs_i, Y_i))$ measures the local SHAP for subject $i$, and $\Phi_j(f, \ObsDist)$ measures global SHAP as the average over subject-level SHAP. We refer to the average importance of variable $j$ over the \textit{empirical} data distribution as $\Phi_j(f, \D) := \frac{1}{n} \sum_{i=1}^n\phi_j(f, (\obs_i, y_i))$. 

We consider the setting in which $p_U$ key variables $\Unobs \in \Uspace \subseteq \mathbb{R}^{p_U}$ are not observed.
Our goal is to quantify variable importance for the true conditional mean function that observes both $\Obs$ and $\Unobs$: $\optmodel(x, u) := \mathbb{E}_{Y \mid X=x, U=u}[Y \mid x,u]$. We define our estimand as $\Phi_j(\optmodel, \UnobsDist) := \mathbb{E}_{(\Obs, \Unobs, Y) \sim \UnobsDist}[\phi_j(\optmodel, (\Obs, \Unobs, Y))]$ where $\UnobsDist$ describes the true distribution from which $(\Obs, \Unobs, Y)$ tuples are sampled. Because $\UnobsDist \neq \ObsDist$ in many contexts, the VI of the population risk minimizer defined over the \textit{observed} variables is not necessarily the same as the VI for the optimal model defined over both observed \textit{and} unobserved variables.

\begin{figure}
    \centering
    \includegraphics[width=0.99\linewidth]{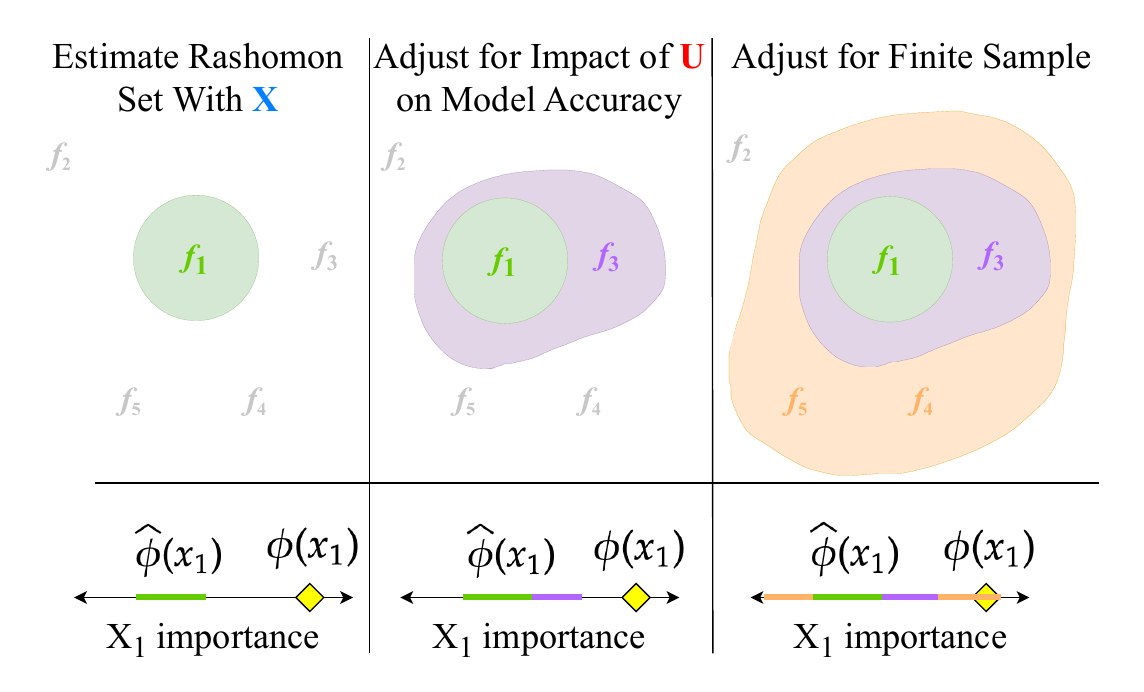}
    \caption{An overview of our framework showing the expansion of the Rashomon set as we account for more sources of error. The $\goldlozenge$ represents the true importance of variable $X_1$ having observed all necessary features. We consider three key factors: the Rashomon Effect, unobserved variables $\Unobs$, and finite sample errors. When adjusting for $\Unobs$, we expand which models are considered as part of the Rashomon set, potentially widening our variable importance interval. Adjusting for finite sample considerations both expands the Rashomon set and expands the range of variable importance values for each model, further widening our intervals.}
    \label{fig:framework}
\end{figure}

If we fix the unobserved variables to a specific quantity, then the only input to $g^*$ that varies is $X$. We define this \textit{conditional sub-model}, conditioned on the unobserved variables being $U = u$, as $\condsubmodel{u}{x} := \mathbb{E}_{Y \mid X=x, U = u}[Y \mid X = x, U = u]$. We will construct bounds on the true variable importance for the optimal model using the \textit{set} of \textit{all} conditional sub-models, defined as $\allcondmodels := \left\{ f_u \mid u \in \Uspace \right\}.$
Throughout this work, we assume that $\Uspace$, $\allcondmodels$, and the model class $\mathcal{F}$, are finite sets for simplicity of notation, although these results can easily be extended to infinite sets. Figure \ref{fig:framework} outlines our framework.

\subsection{Estimating Distribution-Invariant Quantities Over $\allcondmodels$}
Some target functions $\Phi$, such as the coefficient vector of linear models, are \textbf{distribution-invariant}, which means $\Phi(f, \mathcal{P}) = \Phi(f, \mathcal{P}')$ for all $f \in \mathcal{F}$, even if data distributions differ. 
We first discuss how we can construct bounds for distribution-invariant target functions when we have omitted variables and then generalize our analysis to consider \textbf{distribution-dependent} functions (e.g., permutation importance) in Section~\ref{sec:connecting_rset_to_target}.


We capture $\allcondmodels$ by leveraging the \textit{Rashomon set}: the set of \textit{all good models} for some objective.
We define the \textit{population, $\epsilon$-threshold Rashomon set} as
\begin{align}
\begin{split}
    &\poprset(\epsilon; \mathcal{F}, \ell, \lambda, \ObsDist) := \\
    &\{f \in \mathcal{F} :
    \mathbb{E}_{(\Obs, Y) \sim \ObsDist}[\ell(f, \Obs, Y; \lambda)] + \lambda(f) \leq \epsilon \}.
\end{split}
\end{align}
The population Rashomon set is the set of all models $f$ from a model class $\mathcal{F}$ with expected loss $\ell : \mathcal{F} \times \Xspace \times \Yspace \to \mathbb{R}$ with respect to the distribution $\ObsDist$ plus model-level regularization penalty $\lambda:\mathcal{F} \to \mathbb{R}$ below some specified threshold $\epsilon$. When regularization is not applied, we use the ``null regularization,'' which we define as $\lambda_0(f):=0$ $\forall f \in \mathcal{F}.$ For notational convenience, we drop notation for conditioning on $\mathcal{F}, \ell, \ObsDist$ unless necessary.
We apply the following assumption to connect $\allcondmodels$ to $\poprset$:
\begin{definition}
\label{assm:bounded_loss}
    Define $\maxcondloss \geq 0$ such that,
\begin{align}
    \mathbb{E}_{(\Obs, Y) \sim \ObsDist} [\ell(\condsubmodelNoInput{\unobs}, \Obs, Y)] \leq \maxcondloss & \quad \forall \condsubmodelNoInput{u} \in \allcondmodels.
\end{align}
\end{definition}
The value of $\maxcondloss$ bounds the loss of each conditional sub-model ($\condsubmodelNoInput{u}$) over observed features $(\Obs, Y) \sim \ObsDist.$ In practice, we leave it to practitioners to provide a reasonable upper bound on $\maxcondloss.$ This quantity measures the \textit{heterogeneity} of the true conditional mean function across subgroups defined by the unobserved features. If there was no heterogeneity between unobserved groups, then $\condsubmodelNoInput{u}$ would predict $Y$ perfectly given $\Obs$ and $\maxcondloss = 0,$ so a small value for $\maxcondloss$ reflects the belief that most of the important information for this predictive task has been measured.
Given Definition \ref{assm:bounded_loss}, we can use $\maxcondloss$ to connect $\allcondmodels$ to a quantity dependent only on observed features:
\begin{proposition}
\label{prop:poprset_coverage}
    Given Definition~\ref{assm:bounded_loss} and $\allcondmodels \subseteq \mathcal{F}$, we know that $\allcondmodels \subseteq \poprset(\maxcondloss; \lambda_0).$
\end{proposition}
This proposition simplifies our goal: rather than finding $\allcondmodels$ -- which depends on unobservables -- we can instead estimate $\poprset(\maxcondloss)$, a superset of $\allcondmodels$ that depends only on the distribution of observed quantities, given knowledge of an upper bound on $\maxcondloss$. However, $\poprset(\maxcondloss)$ is a population-level quantity, and we only observe a finite sample of data. We define the \textit{empirical, $\epsilon$-threshold Rashomon set} over a set of $n$ samples as:
\begin{align}
\begin{split}
    &\estrset(\epsilon; \mathcal{F}, \ell, \lambda) \\
    &:= \left\{f \in \mathcal{F} : \frac{1}{n} \sum_{i=1}^n\ell(f, \obs_i, y_i) + \lambda(f) \leq \epsilon \right\}.
\end{split}
\end{align}

The empirical Rashomon set contains models whose \textit{empirical loss} is less than some specified threshold $\epsilon$. We can directly compute the empirical Rashomon set for several model classes including decision trees \citep{xin2022exploring} and kernel ridge regression \citep{fisher2019all}. In general, we cannot guarantee that $\poprset(\epsilon; \mathcal{F}, \ell, \lambda_0) = \estrset(\epsilon; \mathcal{F}, \ell, \lambda)$ because of sampling uncertainty and regularization bias. The following theorem provides probabilistic, finite sample bounds connecting $\estrset$ to $\poprset$ by correcting for these issues:

\begin{theorem}
\label{thm:smart_eps}
     For any loss function $\ell$ bounded between $\ell_{\min}$ and $\ell_{\max}$ and $\epsilon \in [\ell_{\min}, \ell_{\max}]$, it holds that 
 \begin{align*}
     &\mathbb{P}\left(\poprset(\epsilon;\mathcal{F}, \ell,\lambda) \subseteq \estrset\left(\smarteps + \epsilon + \lambda_{\sup};\mathcal{F}, \ell, \lambda\right)\right) \geq 1 - \delta, \\
     &\text{ where\;\;}\smarteps =  \sqrt{\frac{(\ell_{\max}-\ell_{\min})^2\ln\left(\frac{C}{\delta} \right)}{2n}}
 \end{align*}
 for sample size $n$, for any value $C \geq |\poprset(\maxcondloss;\mathcal{F}, \ell, \lambda)|$ (e.g., $C:=|\mathcal{F}|$), regularization penalty $\lambda$, and regularization upper bound $\suplambda = \sup_{f \in \mathcal{F}} \lambda(f).$ Moreover, 
\begin{align*}
    &\mathbb{P}\left(\estrset(\epsilon + \varepsilon_n + \lambda_{sup}; \mathcal{F}, \ell, \lambda) \subseteq \mathcal{R}(\epsilon; \mathcal{F}, \ell, \lambda) \right) \\
    &\geq  1 - |\mathcal{F}| \exp\left\{ \frac{-2n\varepsilon_n^2}{(\ell_{\text{max}}-\ell_{\text{min}})^2} \right\}.
\end{align*}
\end{theorem}
\textbf{Theorem \ref{thm:smart_eps} guarantees, with high-probability, that the estimated Rashomon set with Rashomon threshold $\smarteps + \epsilon + \suplambda$ is a  \textit{finite sample} superset of the population $\epsilon$-threshold Rashomon set}. Although our population Rashomon set in Proposition \ref{prop:poprset_coverage} considers an unregularized objective (i.e., $\lambda=\lambda_0$), existing algorithms for estimating Rashomon sets \textit{require} non-zero regularization penalties \citep{xin2022exploring, chen2023understanding}; as such, we also include the worst-case regularization $\suplambda$ in our finite-sample correction. This ensures that a model that would be in the unregularized empirical Rashomon set is not excluded due to regularization bias. For example, if the regularization function is 0.01 times the number of leaves in a tree and $\condsubmodelNoInput{u}$ is a tree with two leaves, we would have $\ell(\condsubmodelNoInput{u}, X, Y; \lambda) = 0 + \lambda(\condsubmodelNoInput{u}) = 0.02$ for loss over $\CondDist{u}$. Even though $\condsubmodelNoInput{u}$ predicts perfectly, it has loss $0.02$, and could be excluded from Rashomon sets if we do not adjust for regularization bias.

The second part of Theorem~\ref{thm:smart_eps} guarantees that our empirical Rashomon set is not a vacuous superset. Specifically, all models that are not in the population Rashomon set are also excluded from our empirical Rashomon set with high probability asymptotically. As our sample size grows, our empirical Rashomon sets converges to the true Rashomon set with high probability, offering a guarantee akin to asymptotic type-II error control but in the context of our Rashomon sets.

Theorem~\ref{thm:smart_eps} is of independent interest to researchers using Rashomon sets in other contexts, like fairness \citep{marx2020predictive}. \textbf{Theorem~\ref{thm:smart_eps} is the first bound connecting empirical and population Rashomon sets that controls type-1 \textit{and} type-2 error by accounting for \textit{both} finite sample \textit{and} regularization biases.}

Because we can construct high-probability supersets for the population Rashomon set, which is a superset of $\allcondmodels$, we know that the empirical Rashomon set with threshold specified in Theorem~\ref{thm:smart_eps} contains all models in $\allcondmodels$ with high probability:

\begin{corollary}
\label{cor:s_coverage}
    Let $\maxcondloss$ be defined as in Definition \ref{assm:bounded_loss}, and assume that $\allcondmodels \subseteq \mathcal{F}$. For any loss function $\ell$ bounded between $\ell_{\min}$ and $\ell_{\max}$, it holds that
 \begin{align*}
     \mathbb{P}\left(\allcondmodels \subseteq \estrset(\smarteps + \maxcondloss + \lambda_{\sup}; \mathcal{F}, \ell, \lambda)\right) \geq 1 - \delta,
 \end{align*}
 for a sample of size $n$ and 
 regularization penalty $\lambda$, with $\smarteps$ defined as in Theorem \ref{thm:smart_eps}.
\end{corollary}

Note that both Theorem~\ref{thm:smart_eps} and Corollary~\ref{cor:s_coverage} define $\smarteps$ using an upper bound on the size of the population Rashomon set, i.e. $C \geq |\poprset(\maxcondloss; \mathcal{F}, \ell, \lambda)|$. It can be difficult to calculate a tight $C$, since in practice we do not know $|\poprset(\maxcondloss; \mathcal{F}, \ell, \lambda)|$. We can guarantee these bounds hold by setting $C:=|\mathcal{F}|,$ but this can yield large empirical Rashomon sets for large model classes. 

In practice, we can often provide tighter bounds by first estimating the size of the population Rashomon set using a separate data split. The following proposition shows that the size of the empirical Rashomon set quickly converges to the size of the population Rashomon set with the same parameters, meaning that we can guarantee \textit{asymptotic}, type-1 error control:

\begin{proposition}
\label{prop:rset_size_mse}
The size of the estimated Rashomon set with threshold $\epsilon' > \epsilon$ is an upper bound on the size of the population Rashomon set with threshold $\epsilon$ with high probability:
\begin{align*}
    \mathbb{P}_{(\Obs, Y) \sim \ObsDist}\left( |\estrset(\epsilon')| > |\poprset(\epsilon)| \right) = 1 - O(n^{-1}).
\end{align*}
\end{proposition}
In the appendix, Corollary~\ref{app-cor:asymp_rset_conv} uses Proposition \ref{prop:rset_size_mse} to show how using an estimated Rashomon set-size will yield a superset of $\allcondmodels$ with high probability asymptotically. Based on this theory, in practice, we use an estimate of the size of the Rashomon set for $C$.

We now use the tools necessary to recover the models from $\allcondmodels$
to guarantee coverage of distribution-invariant variable importance functions across $\allcondmodels.$ We present a general result that provides probabilistic coverage over the variable importance with respect to the entire \textit{observed} data distribution for each model in $\allcondmodels$:
\begin{theorem}
\label{thm:vi_coverage_uncond}
    Let $\vibound$ be a value such that, for all $f \in \allcondmodels,$ 
    $$\mathbb{P}_{\mathcal{D}^{(n)}}\left(\Phi_j(f, \ObsDist) \in \left[ \Phi_j(f, \D) \pm \vibound \right]\right) \geq 1 - \gamma,$$
    where $0\leq \gamma \leq 1$ and $\mathbb{P}_{\mathcal{D}^{(n)}}$ denotes the probability of drawing the observed $n$ samples from $\mathcal{P}_{XY}$. It follows that, for all $f \in \allcondmodels,$ with probability at least $1 - (\delta + \gamma),$
    \begin{align*}
     &\left\{\Phi_j(f, \ObsDist) \mid \condsubmodelNoInput{} \in \allcondmodels \right\} \subseteq \\ 
    &\begin{bmatrix} 
     &\inf_{f' \in \estrset(\smarteps + \maxcondloss + \lambda_{\sup})}\Phi_j(f', \D) - \vibound, \\
     &\sup_{f' \in \estrset(\smarteps + \maxcondloss + \lambda_{\sup})}\Phi_j(f', \D) + \vibound
     \end{bmatrix},
    \end{align*}
    where $\smarteps$ and $\delta$ are defined as in Corollary \ref{cor:s_coverage}. 
\end{theorem}
Theorem \ref{thm:vi_coverage_uncond} uses Corollary \ref{cor:s_coverage} to show that, given a finite sample bound for the estimation of the variable importance $\Phi_j$, 
we can compute an interval that contains the importance of a variable simultaneously for every model in $\allcondmodels$ with high probability. In our experiments, we apply established finite sample bounds for subtractive model reliance \citep{fisher2019all}.

\subsection{Connecting to  $\Phi(\optmodel, \UnobsDist)$} \label{sec:connecting_rset_to_target}
Because many variable importance functions depend heavily on the input distribution, it is not necessarily true that the variable importance for the optimal model will be contained within the bounds on variable importance for conditional submodels evaluated on the \textit{observed} data. This inequality can be seen through a simple application of the law of iterated expectation:
\begin{align*}
    &\Phi_j(\optmodel, \UnobsDist) \\ &= \mathbb{E}_{(\Obs, \Unobs, Y) \sim \UnobsDist}[\phi_j(\optmodel, (\Obs, Y))] \text{ (by definition)}\\
    &= \sum_{u \in \mathcal{U}} \mathbb{P}(U= u) \mathbb{E}_{(\Obs, Y) \sim \CondDist{u}}[\phi_j(\optmodel, (\Obs, Y))] \\
    &\quad \text{ (by law of iterated expectation)}\\
    &= \sum_{u \in \mathcal{U}} \mathbb{P}(U= u) \mathbb{E}_{(\Obs, Y) \sim \CondDist{u}}[\phi_j(\condsubmodelNoInput{u}, (\Obs, Y))] \\
    &\quad \text{ (by definition,  $\optmodel = \condsubmodelNoInput{u}$ when $U = u$)} \\
    &\neq \sum_{u \in \mathcal{U}} \mathbb{P}(U= u) \mathbb{E}_{(\Obs, Y) \sim \ObsDist}[\phi_j(\condsubmodelNoInput{u}, (\Obs, Y))] \\
    &\quad \text{ 
 (because $\CondDist{u} \neq \ObsDist$).}
\end{align*}
That is, the importance of variable $j$ to $g^*$ can be expressed in terms of the variable importance across every $\condsubmodelNoInput{\unobs}$ \textit{with respect to $\CondDist{\unobs}$}, the distribution of observed data \textit{conditional on the unobserved features}. This presents a problem, because we do not know which conditional distribution $\CondDist{\unobs}$ each sample in our observed dataset is drawn from because we do not know $u$. 
We overcome this challenge by bounding how sensitive our variable importance metric is to this kind of distribution shift:
\begin{assumption}
\label{assm:l1_dist}
    Assume that there exists a known $\maxdistshift_j$ such that, for all $u \in \Uspace$, 
    $
    \left|
        \Phi_j(\condsubmodelNoInput{u}, \CondDist{u}) - \Phi_j(\condsubmodelNoInput{u}, \mathcal{P}_{XY|U \neq u})
    \right| \leq \maxdistshift_j.
    $
\end{assumption}
We refer to this quantity $\maxdistshift_j$ as VI-drift. VI-drift is large if the data distribution conditioned on $U=u$ is very different from that conditioned on $U\neq u$ for some $u,$ and if the variable importance metric $\Phi_j$ is sensitive to this difference. We elaborate on this assumption in Appendix \ref{app:assm_elaboration}. 
%
Under Assumption \ref{assm:l1_dist}, we can now build on Theorem \ref{thm:vi_coverage_uncond} to provide finite sample bounds for $\Phi_j(\optmodel, \UnobsDist)$:

\begin{theorem}
\label{thm:optmodel_coverage}
Let $\vibound$ and $\gamma$ be defined as in Theorem \ref{thm:vi_coverage_uncond}, $\smarteps$ and $\delta$ be defined as in Theorem~\ref{cor:s_coverage}, and $\maxdistshift_j$ be defined as in Assumption~\ref{assm:l1_dist}.
With probability at least $1 - (\delta + \gamma), \Phi_j(\optmodel, \UnobsDist)\in$
    $$
    \begin{bmatrix} 
     &\inf_{f \in \estrset(\smarteps + \maxcondloss + \lambda_{\sup})}\Phi_j(f, \D) - \maxdistshift_j - \vibound, \\
     &\sup_{f \in \estrset(\smarteps + \maxcondloss + \lambda_{\sup})}\Phi_j(f, \D) + \maxdistshift_j + \vibound
     \end{bmatrix}.
    $$
\end{theorem}

\begin{figure*}[h!]
    \centering
    \includegraphics[width=0.9\linewidth]{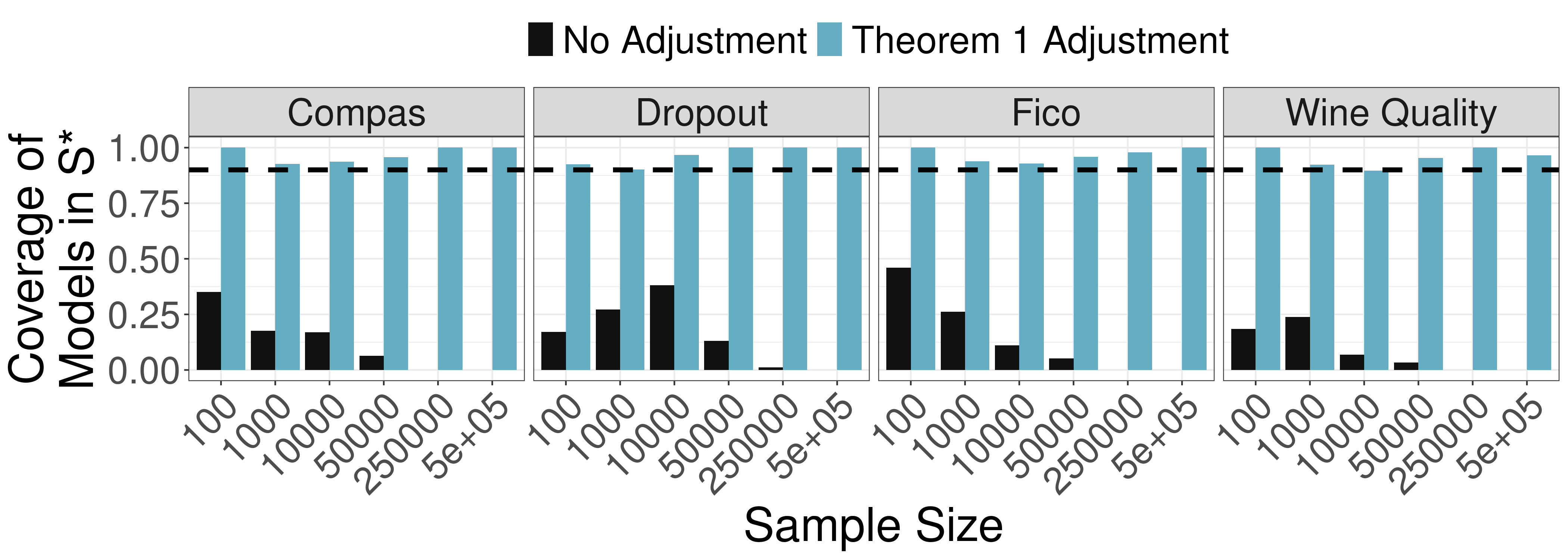}
    \caption{Verifying Theorem~\ref{thm:smart_eps} in finite sample datasets. We compute the proportion of 100 random draws of the each dataset in which Rashomon sets estimated with the Rashomon threshold adjusting for finite sample biases as in Theorem \ref{thm:smart_eps} (in blue) and without any adjustment (in black) captures each $\condsubmodelNoInput{u}$ for each setting. The target coverage rate is $\geq 0.9$, with $\delta=0.1$. Across all sample sizes and datasets, omitting finite sample adjustments yields Rashomon sets that leave out necessary models. In contrast, our adjustment yields the target coverage rate, verifying the theorem holds.  
    We use the estimated Rashomon set size as our upper bound on the size of $\allcondmodels$. } \label{fig:submodel_capture_important_unobs}
\end{figure*}

In Theorem~\ref{thm:optmodel_coverage}, the scalar $\maxdistshift_j$ considers corrections to the model-level variable importance estimate due to differences in distribution. 
If there exists no difference between the distributions (implying that $\maxdistshift_j = 0$), then the variable importance computed for $\condsubmodelNoInput{u}$ on the observed data will be a consistent estimate for $\Phi(\condsubmodelNoInput{u}, \ObsDist) = \Phi(\condsubmodelNoInput{u}, \CondDist{u})$ because $\ObsDist = \CondDist{u},$ and we do not need to perform any correction; if the distributions are very different, $\Phi(\condsubmodelNoInput{u}, \ObsDist)$ could be very different.
Theorem~\ref{thm:optmodel_coverage} combines the finite-sample Rashomon set estimation correction from Theorem~\ref{thm:smart_eps}, the finite-sample variable importance estimation correction from Theorem~\ref{thm:vi_coverage_uncond}, and finally introduces a distribution shift correction. In doing so, \textbf{we guarantee that our bounds contain the true variable importance even with omitted variables with high probability.} Figure \ref{fig:vi_by_loss} (discussed later) contains an example of how this analysis may work in practice.

\section{EXPERIMENTS} \label{sec:experiments}
\subsection{Semi-synthetic Experiments}
\label{subsec:semisynthetic}

\begin{figure*}[h!]
    \centering
    \includegraphics[width=0.9\linewidth]{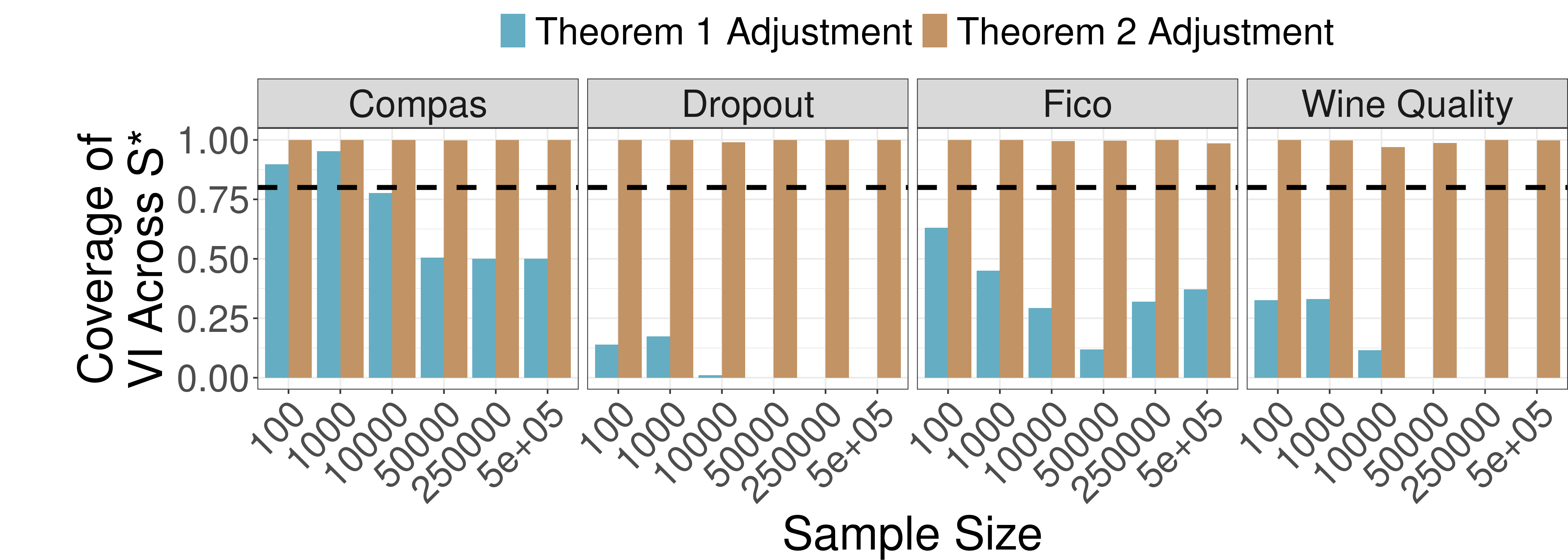}
    \caption{Verifying Theorem~\ref{thm:vi_coverage_uncond}.
    We achieve the specified coverage rate of $\geq0.8$ only when adjusting for (i) model uncertainty via Theorem~\ref{thm:smart_eps} and (ii) variable importance estimation uncertainty (in gold). Adjusting for model uncertainty alone (in blue) is not sufficient. For each setting, we compute the proportion of 100 experiments where our variable importance bounds capture the true variable importance for all submodels $f_u \in \allcondmodels$, averaged over variables. We use an estimate of the true Rashomon set size as our upper bound on the size of $\allcondmodels$.}
    \label{fig:vi_coverage_uncond}
\end{figure*}

We now turn to evaluate each of our primary theoretical claims empirically. Because we cannot know the true variable importance in real observational data, we use ``semi-synthetic'' datasets generated as follows.
Given a tabular dataset $\mathcal{D}^{(n)}$, we first split the dataset into $K$ equally sized partitions $\mathcal{D}^{(n /K)}_1, \mathcal{D}^{(n /K)}_2, \hdots \mathcal{D}^{(n /K)}_K$. 
We then fit a decision tree classifier on each partition sequentially, yielding $K$ distinct models $\condsubmodelNoInput{1}, \condsubmodelNoInput{2}, \hdots \condsubmodelNoInput{K}$. 
In order to ensure that these models are distinct, when fitting the $k$-th tree, we set all entries of each feature used by the previously fitted models $\condsubmodelNoInput{1}, \condsubmodelNoInput{2}, \hdots \condsubmodelNoInput{k-1}$
equal to 0. 
For each partition, we create a semi-synthetic dataset where each label is replaced by the prediction of the corresponding model, i.e., $\bar{\mathcal{D}}^{(n/K)}_k := \{(X_i, \condsubmodel{k}{X_i})\}^{k n/K}_{i=(k-1)n/K}$.
Finally, we combine the semi-synthetic datasets back into one.
In doing so, we treat the partition to which each sample was assigned as an important unobserved feature $U \in \mathcal{U} := \{1, 2, \hdots, K\}$; if the $i$-th sample was assigned to partition $k$, we say $u_i = k$.
This yields a known true conditional mean function $\optmodel(x_i, u_i) := \condsubmodel{u_i}{x_i}, u_i \in \{1, 2, \hdots, K\}$. 
This setup allows us to know each ground truth quantity of interest ($\optmodel$ and $\allcondmodels$ with their corresponding variable importance values, $\maxcondloss$) while working with a fairly realistic distribution for $X$. 

In the following experiments, we apply this procedure to four disparate real-world datasets with $K=2$:
(1) Compas \citep{propublica}, which concerns criminal recidivism prediction for 6,907 individuals;
(2) Dropout \citep{martins2021early}, which concerns predicting whether students will drop out of college for 4,424 individuals;
(3) FICO \citep{competition}, which concerns predicting whether an individual will repay line of credit within 2 years for 10,459 individuals;
(4) Wine Quality \citep{cortez2009modeling}, which concerns predicting wine quality ratings over 6,497 wines.

For each dataset, we evaluate our framework using random draws with replacement of sample sizes 100; 1,000; 10,000; 50,000; 250,000; and 500,000 from $\bar{\mathcal{D}}^{(n)}$, each over 100 iterations. 
We consider only the first 8 features (except Compas, which has 7 features) from each dataset to enable fast computation of the Rashomon set.
In each setting, we consider the model class $\mathcal{F}$ of depth $3$ or less decision trees found using TreeFarms \cite{xin2022exploring}, and similarly restrict each $\condsubmodelNoInput{k}$ to have a depth of at most $3$. We repeat each of these experiments using an approximation of the Rashomon set of random forests \citep{laberge2023partial} in Appendix \ref{app:random_forest_results}. To faithfully evaluate the theoretical claims above, we use the true value for $\maxcondloss$ and $\maxdistshift_j$ in each experiment.

\begin{figure*}[h!]
    \centering
    \includegraphics[width=0.9\linewidth]{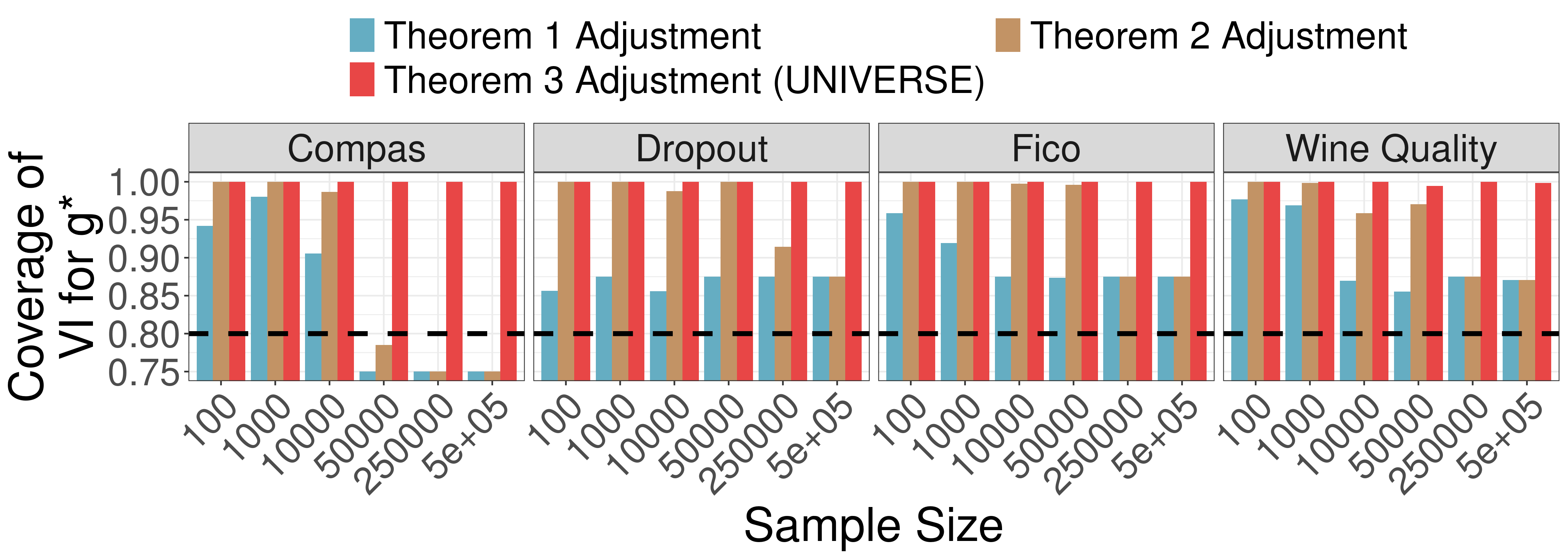}
    \caption{Verifying Theorem~\ref{thm:optmodel_coverage}.
    We consistently achieve the specified coverage rate of $\geq 0.8$ only when we account for (i) model uncertainty, (ii) variable importance uncertainty, and (iii) VI drift. 
    Each bar measures the proportion of 100 experiments in which our bounds capture the true variable importance for the true model $\optmodel$. Plots are colored such that blue only accounts for finite sample model uncertainty as in Theorem~\ref{thm:smart_eps}, gold also adjusts for uncertainty in estimating subtractive model reliance (MR) at the model-level as in Theorem~\ref{thm:vi_coverage_uncond}, and red adjusts for the previous two \textit{and} distribution shifts induced by omitted variables \ref{thm:optmodel_coverage}. All three adjustments are necessary to achieve the target coverage rate of $\geq 0.8$, with $\delta=\gamma=0.1$.}
    \label{fig:dgp_coverage}
\end{figure*}
\paragraph{Finding each model in $\allcondmodels$} 

Corollary \ref{cor:s_coverage} demonstrates how we may select a value $\varepsilon_n$ such that, with high probability, the empirical Rashomon set contains $\allcondmodels$. We first evaluate the validity of this Corollary with a target confidence of $1 - \delta = 0.9$.

Figure \ref{fig:submodel_capture_important_unobs} presents the results of this evaluation. We find that, across all six sample sizes and all four reference datasets, omitting finite sample and regularization adjustments yields Rashomon sets that leave out necessary models. In contrast, \textbf{our empirical Rashomon set based on Corollary \ref{cor:s_coverage} contains every model in $\allcondmodels$ at above the specified rate}. These results demonstrate that adjusting for finite sample and regularization biases is necessary for constructing high-probability, finite-sample covers of the population Rashomon sets. Even at large sample sizes, not adjusting for these factors yields \textit{under} coverage because of the regularization that existing algorithms necessitate to compute Rashomon sets. Specifically, $\maxcondloss$ is defined as the maximum expected \textit{unregularized} loss for any model in $\allcondmodels$. As such, the expected \textit{regularized} loss for the highest-loss model in $\allcondmodels$ will exceed $\maxcondloss$ and the highest-loss model is therefore consistently omitted from the empirical Rashomon set. 

\paragraph{Recovering variable importance for each model in $\allcondmodels$} We now turn to evaluate Theorem~\ref{thm:vi_coverage_uncond}: given a finite sample bound for the estimation of our variable importance metric for a specified model, we can cover the importance of each variable to each model in $\allcondmodels$ with at least a specified probability. In this experiment, we set our target probability to $1-(\delta + \gamma) = 0.8$, and apply the finite sample bound for subtractive model reliance (MR) from Equation B.26 of \citet{fisher2019all}.
Figure \ref{fig:vi_coverage_uncond} demonstrates that adjusting only for model uncertainty as in Corollary~\ref{thm:smart_eps} yields invalid bounds on model-level variable importance. In fact, for sample sizes larger than 10,000 observations, omitting variable importance uncertainty quantification yields intervals that \textit{never} contain the true variable importance for all conditional submodels. In contrast, 
\textbf{across all four datasets and at all sample sizes, our approach achieves nominal coverage.}

\paragraph{Recovering variable importance for the $\optmodel$ with unobserved features} Finally, we evaluate Theorem \ref{thm:optmodel_coverage}, which provides high probability bounds on variable importance to the true model $\optmodel$, even with unobserved variables. We again set our target coverage rate to $1-(\delta + \gamma) = 0.8$ and apply the finite sample bound for subtractive MR \citet{fisher2019all}. This reflects applying the entire \ourmethodAcronym{} framework.

Figure \ref{fig:dgp_coverage} demonstrates that \textbf{we consistently cover the true importance to $\optmodel$ in more than the specified $80\%$ of cases across all four datasets}, even though $\optmodel$ depends on unobserved variables. Moreover, Figure \ref{fig:dgp_coverage} demonstrates that each component of Theorem \ref{thm:optmodel_coverage} is necessary to achieve this coverage rate. Without these adjustments, our bounds would substantially undercover the true variable importance, as highlighted on the Compas dataset.






\subsection{Case Study}
\label{subsec:real_data}
We use our framework to study credit risk assessment using a dataset developed by the Fair Isaac Corporation (FICO), focusing on the most powerful predictive feature: \texttt{External Risk Estimate} (ERE).
ERE is a risk score developed by external agencies like FICO for which banks often pay when evaluating credit applications. 
After controlling for other observed features, ERE remains important to \textit{every} nearly-optimal model, as shown in Figure~\ref{fig:real_intervals_intro}. 
However, the FICO dataset does not share information about features that are likely predictive of loan default, like income. 
If these unmeasured features could easily explain the information in ERE, banks could collect alternative 
features that would better serve customers. \ourmethodAcronym~allows analysts to answer the following: how much signal would another feature need to capture to replace ERE?

\begin{figure}[th!]
    \centering
    \includegraphics[width=0.95\linewidth]{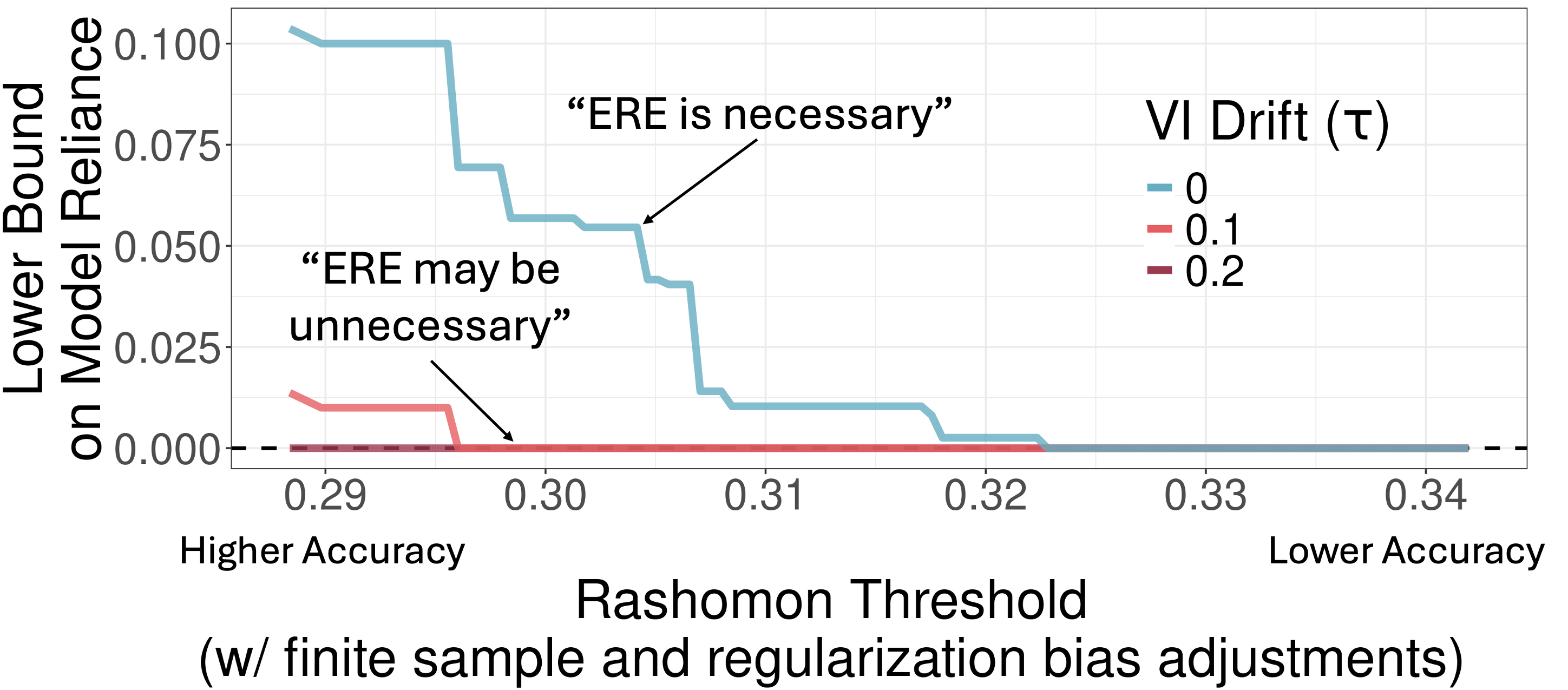}
    \caption{An example of how our tool may be used in practice. The x-axis displays the Rashomon threshold used for estimating Rashomon sets (adjusting for finite sample and regularization biases as in Theorem~\ref{thm:smart_eps}). The y-axis displays the range of subtractive model reliance (MR) we compute adjusting for model uncertainty, variable importance estimation uncertainty, and distribution shifts as in Theorem~\ref{thm:optmodel_coverage} as we vary the Rashomon threshold. Each color shows the range as the assumed amount of variable importance drift from distribution shift $\maxdistshift_j$ increases. The x-axis begins at the lowest achievable loss in the dataset.}
    \label{fig:vi_by_loss}
\end{figure}

We use the Rashomon set of sparse decision trees with 0-1 loss and the same hyperparameters as in Section~\ref{sec:experiments} to compute \ourmethodAcronym.
This is an appropriate model class because sparse decision trees have achieved similar performance to more complex model classes like boosted trees on this dataset \citep{rudin2024amazing}.

Figure~\ref{fig:vi_by_loss} displays the results of our analyses. The x-axis displays the Rashomon threshold we used for estimating Rashomon sets (adjusting for finite sample and regularization biases as in Theorem~\ref{thm:smart_eps}), with larger x-values reflecting a large belief about the value of $\maxcondloss$. The y-axis displays the lower bound on subtractive model reliance accounting for finite sample uncertainty and distribution shifts. Each color displays results under different amounts of assumed VI drift, $\maxdistshift_j$.

Setting $\maxdistshift_j=0$ means that conditioning on an unobserved variable does not change the distributions of ERE or outcome in a way that affects variable importance.
Even under this assumption, we find that we would only need to identify a feature whose conditional sub-models achieve a 0-1 loss of $\maxcondloss=0.32$ (which corresponds to an accuracy of 68\%) to find a model that does not depend on ERE; our current best model achieves a loss of 0.29 (accuracy of 71\%). 
This accuracy requirement becomes even smaller as we allow $\maxdistshift_j$ to increase, suggesting that ERE's importance is not robust to a moderate degree of unobserved confounding. 
If we depended only on the observed features, an analyst may conclude that ERE is really important and must always be considered to predict credit risk. 
However, our analysis suggests that other, more interpretable features could easily replace ERE, enabling the bank to give more meaningful feedback to loan applicants.

\section{CONCLUSION} \label{sec:conclusion}

In this work, we introduced \ourmethodAcronym, the first variable importance method to account for finite sample concerns, the Rashomon effect, and unobserved variables. We proved both theoretically and empirically that \ourmethodAcronym{} can recover the true variable importance to the true underlying conditional mean function, even in the presence of these sources of error.


The \ourmethodAcronym{} framework is general and applicable to \textit{any} model class. However, finding the complete Rashomon set is often computationally intensive. Future work should extend our framework to settings in which we \textit{approximate} the complete Rashomon set of more complex model classes like neural networks \citep{donnelly2025rashomon} or deep decision trees \citep{babbar2025near}.
Additionally, we develop \ourmethodAcronym{} for classification problems, but many problems require modeling continuous or survival outcomes. 
Applied practitioners may benefit from generalizations of \ourmethodAcronym{} for these outcomes. 
Nonetheless, \ourmethodAcronym{} represents a substantial step in using VI in challenging, realistic settings with unobserved confounders.

\acknowledgments{
    Support from Bocconi Senior Researchers’ Grant, 2024, 603020 is gratefully acknowledged by Emanuele Borgonovo. Srikar Katta is supported by the Apple Scholar in AI/ML Fellowship. This material is based upon work supported by the National Science Foundation Graduate Research Fellowship under grant number DGE 2139754, as well as National Institutes of Health/NIDA grant number R01DA054994. We also gratefully acknowledge support from the National Science Foundation FAI under grant number 2147061 and an Amazon FAI Gift. Thank you to the anonymous reviewers whose feedback greatly improved this manuscript.
}

\bibliographystyle{apalike}
\bibliography{references}

\clearpage
\section*{CHECKLIST}



 \begin{enumerate}

 \item For all models and algorithms presented, check if you include:
 \begin{enumerate}
   \item A clear description of the mathematical setting, assumptions, algorithm, and/or model. \textbf{Yes} -- we explicitly state the assumptions behind our theory.
   \item An analysis of the properties and complexity (time, space, sample size) of any algorithm. \textbf{Not Applicable} -- our framework inherits the runtime of the underlying Rashomon set computation algorithm, and as such is not interesting to provide in general.
   \item (Optional) Anonymized source code, with specification of all dependencies, including external libraries. \textbf{Yes}
 \end{enumerate}

 \item For any theoretical claim, check if you include:
 \begin{enumerate}
   \item Statements of the full set of assumptions of all theoretical results. \textbf{Yes}
   \item Complete proofs of all theoretical results. \textbf{Yes} -- these are provided in Appendix \ref{app:proofs}.
   \item Clear explanations of any assumptions. \textbf{Yes} -- we contextualize each assumption with illustrative examples.   
 \end{enumerate}

 \item For all figures and tables that present empirical results, check if you include:
 \begin{enumerate}
   \item The code, data, and instructions needed to reproduce the main experimental results (either in the supplemental material or as a URL). \textbf{Yes} -- these are available in the GitHub link.
   \item All the training details (e.g., data splits, hyperparameters, how they were chosen). \textbf{Yes} -- hyperparameters are described in the text, and data splits can be found through the code.
         \item A clear definition of the specific measure or statistics and error bars (e.g., with respect to the random seed after running experiments multiple times). \textbf{Yes}.
         \item A description of the computing infrastructure used. (e.g., type of GPUs, internal cluster, or cloud provider). \textbf{Yes}
 \end{enumerate}

 \item If you are using existing assets (e.g., code, data, models) or curating/releasing new assets, check if you include:
 \begin{enumerate}
   \item Citations of the creator If your work uses existing assets. \textbf{Yes}
   \item The license information of the assets, if applicable. \textbf{Not Applicable} -- public code packages and datasets were used as existing assets, and as such each have license information available.
   \item New assets either in the supplemental material or as a URL, if applicable. \textbf{Yes}
   \item Information about consent from data providers/curators. \textbf{Not Applicable} -- only public data was used.
   \item Discussion of sensible content if applicable, e.g., personally identifiable information or offensive content. \textbf{Not Applicable}
 \end{enumerate}

 \item If you used crowdsourcing or conducted research with human subjects, check if you include:
 \begin{enumerate}
   \item The full text of instructions given to participants and screenshots. \textbf{Not Applicable}
   \item Descriptions of potential participant risks, with links to Institutional Review Board (IRB) approvals if applicable. \textbf{Not Applicable}
   \item The estimated hourly wage paid to participants and the total amount spent on participant compensation. \textbf{Not Applicable}
 \end{enumerate}

 \end{enumerate}

\appendix
\onecolumn
\aistatstitle{Doctor Rashomon and the UNIVERSE of Madness: Variable Importance with Unobserved Confounding and the Rashomon Effect (Appendix)}

\section{NOTATION/DEFINITIONS}


\begin{table}[h]
    \centering
    \begin{tabularx}{\textwidth}{c|X}
        Notation & Definition \\
        \hline 
        $\D$ & The observed dataset of $n$ samples\\
        $\obs$ & Observed covariates for a single observation \\
        $\unobs$ & Unobserved covariates for a single observation \\
        $y$ & The outcome of interest \\
        $\optmodel$ & The conditional mean function of $y$ given \textit{all} covariates\\
        $\condsubmodelNoInput{u}$ & The conditional mean function of \textit{observed} variables given the fixed values of the \textit{unobserved} variables to $u$\\
        $\allcondmodels$ & The set of all ground-truth conditional submodels \\
        $\poprset$ & The population rashomon set \\
        $\estrset$ & The rashomon set estimated from $n$ observations\\
        $\Xspace$ & The space of \textit{observed} covariates \\
        $\Uspace$ & The space of \textit{unobserved} covariates \\
        $\Yspace$ & The space of outcomes \\
        $\maxcondloss$ & The (assumed) maximum expected loss over conditional submodels when applied to the full population\\
        $\delta$ & The allowed probability that a model in the population Rashomon set is not in our estimated Rashomon set \\
        $\smarteps$ & Finite sample epsilon correction to insure our estimated Rashmon Set contains the population Rashomon set with high probability\\
        $\gamma$ & The allowed probability that the estimated variable importance for a model is outisde our predicted bounds on variable importance \\
        $\suplambda$ & Worst case regularization possible from model class \\
        $\maxdistshift$ & The maximum shift in variable importance over conditional submodels due to conditioning on a given value of the unobserved covariates \\
        $\maxdistshift_j$ & The maximum shift in variable importance over conditional submodels due to conditioning on a given value of the unobserved covariates \\
        $\phi_j$ & The variable importance for variable $j$, with respect to a single sample \\
        $\Phi_j$ & The expectation of $\phi_j$ over samples from the data distribution \\
        \hline
    \end{tabularx}
    \caption{A brief summary of the notation used in this work.}
    \label{tab:my_label}
\end{table}

\newpage
\section{MOTIVATIONAL EXAMPLE}
\label{app:motivational_example}
In this section, we provide a brief example where unobserved confounding leads to misleading results. Let $X_1, X_2, X_3, U \sim \text{Bernoulli}(p=0.5),$ and our label be generated as:
\begin{align*}
    Y = \begin{cases}
        X_1 \text{ XOR } X_2 & \text{if } U =0 \\
        X_1 \text{ XOR } X_3 & \text{otherwise } 
    \end{cases}
\end{align*}

Let our model class of interest $\mathcal{F}$ consist of universal approximators, and the algorithm considered $\mathcal{A}$ always produce an optimal model for the given data. We will consider 0-1 loss (or, equivalently in this case, MSE).

If we observe $X_1, X_2, X_3,$ and allow $n \to \infty,$ note that there are four settings of $X_1, X_2, X_3$ which contain two distinct values for $Y$. Thus, the Bayes' optimal accuracy over this data is $75\%,$ and any model that achieves this accuracy is optimal. Two such models are:
\begin{align*}
    f_1(X_1, X_2, X_3) = X_1 \text{ XOR } X_2\\
    f_2(X_1, X_2, X_3) = X_1 \text{ XOR } X_3
\end{align*}

Thus, $\mathcal{A}(X_1, X_2, X_3)$ may produce either of $f_1, f_2$. We therefore have that
\begin{align*}
    LOCO(X_2) &= \ell(\mathcal{A}(X_1, X_3, Y)) - \ell(\mathcal{A}(X_1, X_2, X_3, Y))\\
    &= \ell(f_2) - \ell(f_1)\\
    &= 0
\end{align*}
and
\begin{align*}
    LOCO(X_3) &= \ell(\mathcal{A}(X_1, X_2, Y)) - \ell(\mathcal{A}(X_1, X_2, X_3, Y))\\
    &= \ell(f_1) - \ell(f_2)\\
    &= 0,
\end{align*}

even though $X_2$ and $X_3$ \textit{are} used by the DGP. In fact, a wide class of measures of statistical association demonstrate a similar behavior because \textit{our outcome is marginally independent of each variable}. By the definition of conditional probability, we have:



\begin{align*}
    P(Y = y| X_1 = x_1) &= \frac{P(Y = y, X_1 = x)}{P(X_1 = x)}\\
    &= \frac{4 / 16}{8/16} \\
    &= P(Y = y)\\
    P(Y = y| X_2 = x_2) &= \frac{P(Y = y, X_2 = x_2)}{P(X_2 = x_2)}\\
    &= \frac{4 / 16}{8/16}\\ 
    &= P(Y = y)\\
    P(Y = y| X_3 = x_3) &= \frac{P(Y = y, X_3 = x_3)}{P(X_3 = x)}\\
    &= \frac{4 / 16}{8/16} \\
    &= P(Y = y)
\end{align*}

Consider an arbitrary measure of statistical association $M$ satisfying the ``zero-independence'' property (i.e., $M(X, Y) = 0$ if and only if $X \perp Y$). Because $X_2 \perp Y,$ $X_2 \perp Y,$ and $X_3 \perp Y,$ we have that $M(X_1, Y) = M(X_2, Y) = M(X_3, Y) = 0.$ 

\newpage
\section{ELABORATION ON ASSUMPTION~\ref{assm:l1_dist}}
\label{app:assm_elaboration}
\setcounter{assumption}{1}
In the main paper, we introduced the following assumption:
\begin{assumption}
    Assume that there exists a known $\maxdistshift_j$ such that, for all $u \in \Uspace$, 
    $$
    \left|
        \Phi_j(\condsubmodelNoInput{u}, \CondDist{u}) - \Phi_j(\condsubmodelNoInput{u}, \mathcal{P}_{XY|U \neq u})
    \right| \leq \maxdistshift_j.
    $$
\end{assumption}
In this section, we elaborate on the meaning of the term $|\Phi_j(\condsubmodelNoInput{u}, \CondDist{u}) - \Phi_j(\condsubmodelNoInput{u}, \mathcal{P}_{XY|U \neq u})|$ through an example. Consider the data distribution reflected in Table \ref{tab:assm_2_example} with each row of the table having equal probability of being drawn. In Table \ref{tab:assm_2_example}, the first four rows have $U=0$, and as such $Y = f_0(X) := X_1$. In the second four, we have $U=1$ and $Y = f_1(X):= X_1 \neq X_2$. In this setting, $U$ is unobserved. Because there are only two values for $U$ to take, we will compute quantities for $\CondDist{0}$ using the first four rows, and $\mathcal{P}_{XY|U \neq 0}$ using the second four.

We will step through computing the subtractive model reliance and 0-1 loss for $f_0$. First, by definition, we know
\begin{align*}
    \Phi_1(f_0, \CondDist{u}) 
    &= \mathbb{E}_{\CondDist{u}}[\mathbb{E}_{X_{i'} \sim \mathcal{P}_X}[\ell(f_0, swap_1(X_i, X_{i'}), Y_i)]] - \mathbb{E}_{\CondDist{u}}[\ell(f_0,X_i,Y_i)], 
\end{align*}
where $swap_j(A, B)$ is a function replacing the $j$-th entry of $A$ with the $j$-th entry of $B$. Because we make perfect predictions for the subpopulation with $U = 0$ using the model $f_0$, we know that $\mathbb{E}_{\CondDist{u}}[\ell(f_0,X_i,Y_i)] = 0$ and $\Phi_1(f_0, \CondDist{u}) = \mathbb{E}_{\CondDist{u}}[\mathbb{E}_{X_{i'} \sim \mathcal{P}_X}[\ell(f_0, swap_1(X_i, X_{i'}), Y_i)]].$

The term $\mathbb{E}_{\CondDist{u}}[\mathbb{E}_{X_{i'} \sim \mathcal{P}_X}[\ell(f_0, swap_1(X_i, X_{i'}), Y_i)]]$ is slightly onerous to compute by hand, but can be intuited from the construction of the problem. Within $\CondDist{u}$, there is equal probability that $X_1 = 0$ and $X_1 = 1$. Because $f_0$ simply returns the value of $X_1$, drawing a random value $X'_1$ from the marginal distribution of $X_1$ will have equal likelihood of producing $f_0(X')=0$ and $f_0(X')=1$. Thus, given the fact that $\mathbb{P}(Y=0)=\mathbb{P}(Y=1)=0.5,$ we have $\mathbb{E}_{\CondDist{u}}[\mathbb{E}_{X_{i'} \sim \mathcal{P}_X}[\ell(f_0, swap_1(X_i, X_{i'}), Y_i)]] = 0.5$, yielding $\Phi_1(f_0, \CondDist{u}) = 0.5$

We now turn to compute $\Phi_j(\condsubmodelNoInput{u}, \mathcal{P}_{XY|U \neq u})$. In this case, we consider the bottom four rows of Table \ref{tab:assm_2_example}, yielding:
\begin{align*}
    \Phi_j(\condsubmodelNoInput{u}, \mathcal{P}_{XY|U \neq u}) 
    &= \mathbb{E}_{\mathcal{P}_{XY|U \neq u}}[\mathbb{E}_{X_{i'} \sim \mathcal{P}_X}[\ell(f_0, swap_1(X_i, X_{i'}), Y_i)]] - \mathbb{E}_{\mathcal{P}_{XY|U \neq u}}[\ell(f_0,X_i,Y_i)] \\
    &= \mathbb{E}_{\mathcal{P}_{XY|U \neq u}}[\mathbb{E}_{X_{i'} \sim \mathcal{P}_X}[\ell(f_0, swap_1(X_i, X_{i'}), Y_i)]] - 0.5
\end{align*}
The term $\mathbb{E}_{\mathcal{P}_{XY|U \neq u}}[\mathbb{E}_{X_{i'} \sim \mathcal{P}_X}[\ell(f_0, swap_1(X_i, X_{i'}), Y_i)]]$ is similarly annoying to compute by hand, but can be intuited from the construction of the problem by the exact same reasoning. Within $\mathcal{P}_{XY|U \neq u}$, there is equal probability that $X_1 = 0$ and $X_1 = 1$. Because $f_0$ simply returns the value of $X_1$, drawing a random value $X'_1$ from the marginal distribution of $X_1$ will have equal likelihood of producing $f_0(X')=0$ and $f_0(X')=1$. Thus, given the fact that $\mathbb{P}(Y=0)=\mathbb{P}(Y=1)=0.5,$ we have $\mathbb{E}_{\mathcal{P}_{XY|U \neq u}}[\mathbb{E}_{x_{i'} \sim \mathcal{P}_X}[\ell(f_0, swap_1(x_i, x_{i'}), y_i)]] = 0.5$. However, in this case the initial loss for $f_0$ was poorer; as such, we have $\Phi_1(f_0, \mathcal{P}_{XY|U \neq u}) = 0.5 - 0.5 = 0$.

The shift in the distribution of $Y$ given $U = 1$ vs $U = 0$ changes the importance of $X_1$ for the conditional submodel $f_0$. When this effect happens, $\tau_j$ is large, because it measures the change in variable importance for conditional submodels across different subgroups of unobserved features. Similarly, if the distribution of $X_1$ changed substantially between conditions, we would expect $\tau_j$ to be larger.

\begin{table}[]
    \centering
    \begin{tabular}{c|c|c|c|c|c}
        $X_1$ & $X_2$ & $U$ & $f_0(X)$ & $f_1(X)$ & Y\\
        \hline
        \hline
        0 & 0 & 0 & 0 & 0 & 0\\
        1 & 0 & 0 & 1 & 1 & 1\\
        0 & 1 & 0 & 0 & 1 & 0\\
        1 & 1 & 0 & 1 & 0 & 1\\
        \hline
        0 & 0 & 1 & 0 & 0 & 0\\
        1 & 0 & 1 & 1 & 1 & 1\\
        0 & 1 & 1 & 0 & 1 & 1\\
        1 & 1 & 1 & 1 & 0 & 0
    \end{tabular}
    \caption{The full distribution of data used for the example below. Each row in this table has equal probability of occurring. In this example, $f_0(X) := X_1$ and $f_1(X):= X_1 \neq X_2$}
    \label{tab:assm_2_example}
\end{table}

\newpage
\section{APPLICABLE VARIABLE IMPORTANCE METRICS}
\label{app:applicable_vi_metrics}
Our framework primarily considers variable importance metrics that can be expressed as the expectation of some sample-wise variable importance quantity. Here, we list two example variable metrics that fit this definition.

\paragraph{Subtractive Model Reliance} is defined by \cite{fisher2019all}  as 
\begin{align*}
    \Phi_j(f, \ObsDist) = \mathbb{E}_{\ObsDist}[\mathbb{E}_{x_{i'} \sim \mathcal{P}_X}[\ell(f, swap_j(x_i, x_{i'}), y_i)]] - \mathbb{E}_{\ObsDist}[\ell(f,x_i,y_i)]
\end{align*}
where 
$$swap_j(x_i, x_{i'}):= \begin{bmatrix}
    x_{i,1} & x_{i,2} & \hdots x_{i',j}  & \hdots x_{i,p}
\end{bmatrix}^T$$
is a function that replaces the $j-$th entry of $x_i$ with that of $x_{i'}$.
We can express this as the expectation of a sample-wise variable importance quantity as follows
\begin{align*}
    \Phi_j(f, \ObsDist) &= \mathbb{E}_{\ObsDist}[\mathbb{E}_{x_{i'} \sim \mathcal{P}_X}[\ell(f, swap_j(x_i, x_{i'}), y_i)]] - \mathbb{E}_{\ObsDist}[\ell(f,x_i,y_i)]\\
    &= \mathbb{E}_{\ObsDist}[\mathbb{E}_{x_{i'} \sim \mathcal{P}_X}[\ell(f, swap_j(x_i, x_{i'}), y_i)] - \ell(f,x_i,y_i)]\\
    &= \mathbb{E}_{\ObsDist}\left[\phi_j(f, x_i, y_i)\right]
\end{align*}
for $\phi_j(f, x_i, y_i):=\mathbb{E}_{x_{i'} \sim \mathcal{P}_X}[\ell(f, swap_j(x_i, x_{i'}), y_i)] - \ell(f,x_i,y_i) $.

\paragraph{SHAP} is defined by \cite{lundberg2017unified} as a sample-wise variable importance metric. Let $\phi_j^{SHAP}(f, x_i, y_i)$ denote the SHAP importance for feature $j$ on sample $i$. In the SHAP package, global importance values are obtained by taking the mean absolute value over $i$; that is, $\Phi_j(f, \mathcal{D}):=\mathbb{E}[\vert \phi_j^{SHAP}(f, x_i, y_i)\vert ].$ Similar reasoning applies to a wide range of extensions of SHAP.

\paragraph{Average Integrated Gradients} Suppose our function class is composed of only differentiable functions $\mathcal{F}.$ Choose some $f \in \mathcal{F}.$ As defined in \citet{sundararajan2017axiomatic}, for a given sample $(x_i, y_i)$, its integrated gradients is computed as
\begin{equation*}
    \phi_j(f, (x_i, y_i)) = (x_i - x_i'))\int_0^1 \frac{\partial f(x_i + \alpha(x_i - x_i')}{dx_i}d\alpha
\end{equation*}
for some alternative covariate profile $x_i'.$ To measure the importance of only feature $j$, we can define $x_{i,j}' = x_{i,j} + \Delta$ for some fixed $\Delta$ and keep all other features $x_{i,j'}' = x_{i,j}$ fixed.
We can now measure the average importance of feature $j$ as an expectation of sample-level integrated gradients over the empirical or population distribution.

\newpage
\section{PROOFS}
\label{app:proofs}

\setcounter{proposition}{0}
\setcounter{theorem}{0}
\setcounter{corollary}{0}
\begin{proposition}
\label{supp_prop:poprset_coverage}
    Given Definition \ref{assm:bounded_loss}, we know that $\allcondmodels \subseteq \poprset(\maxcondloss; \lambda_0).$
\end{proposition}
\begin{proof}
    Definition \ref{assm:bounded_loss} states that, for some $\maxcondloss > 0$,
\begin{align}
    \mathbb{E}_{(\Obs, Y) \sim \ObsDist} [\ell(\condsubmodelNoInput{\unobs}, \Obs, Y)] \leq \maxcondloss & \quad \forall \condsubmodelNoInput{u} \in \allcondmodels.
\end{align}
The Rashomon set with regularization $\lambda_0$ is defined as 
\begin{align*}
    \poprset(\epsilon; \mathcal{F}, \ell, \lambda_0, \ObsDist) := 
    \{f \in \mathcal{F} :
    \mathbb{E}_{(\Obs, Y) \sim \ObsDist}[\ell(f, \Obs, Y; \lambda) + 0\leq \epsilon] \}.
\end{align*}
Thus, $\allcondmodels \subseteq \poprset(\maxcondloss; \lambda_0)$ by the definition of $\poprset(\maxcondloss; \lambda_0)$.
\end{proof}

\begin{theorem}[Recovering the Population Rashomon Set]
\label{app-thm:smart_eps}
     For any loss function $\ell$ bounded between $\ell_{\min}$ and $\ell_{\max}$ and $\epsilon \in [\ell_{\min}, \ell_{\max}]$, it holds that 
 \begin{align*}
     &\mathbb{P}\left(\poprset(\epsilon;\mathcal{F}, \ell,\lambda) \subseteq \estrset\left(\smarteps + \epsilon + \lambda_{\sup};\mathcal{F}, \ell, \lambda\right)\right) \geq 1 - \delta, \text{ where\;\;}\smarteps =  \sqrt{\frac{(\ell_{\max}-\ell_{\min})^2\ln\left(\frac{C}{\delta} \right)}{2n}}
 \end{align*}
 for sample size $n$, for any value $C \geq |\poprset(\maxcondloss;\mathcal{F}, \ell, \lambda)|$ (e.g., $C:=|\mathcal{F}|$), regularization penalty $\lambda$, and regularization upper bound $\suplambda = \sup_{f \in \mathcal{F}} \lambda(f).$

Moreover,
\begin{align*}
    &\mathbb{P}\left(\estrset(\epsilon + \varepsilon_n + \lambda_{sup}; \mathcal{F}, \ell, \lambda) \subseteq \mathcal{R}(\epsilon; \mathcal{F}, \ell, \lambda) \right) \geq  1 - |\mathcal{F}| \exp\left\{ \frac{-2n\varepsilon_n^2}{(\ell_{\text{max}}-\ell_{\text{min}})^2} \right\}.
\end{align*}
\end{theorem}
\begin{proof}
Throughout this proof, we denote the expected loss $\mathbb{E}_{X, Y \sim \ObsDist}\ell(f, X, Y)$ simply as $\ell(f)$ and the empirical loss $\frac{1}{n} \sum_{i=1}^n \ell(f, x_i, y_i)$ as $\hat{\ell}(f)$. All probabilities are with respect to $X, Y \sim \ObsDist$. 
First, recall that we can form the following uniform bound:
\begin{align*}
    &\mathbb{P}\left(\exists f \in \mathcal{R}(\epsilon; \mathcal{F}, \ell, \lambda) \text{ s.t. } f \notin \estrset(\epsilon + \varepsilon_n + \lambda_{sup}; \mathcal{F}, \ell, \lambda) \right)\\ 
    &= \mathbb{P}\left(\exists f \in \mathcal{R}(\epsilon) \text{ s.t. } \hat{\ell}(f) + \lambda(f) > \epsilon + \varepsilon_n + \lambda_{sup} \right) \text{ by definition of Rashomon set exclusion}  \\
    &\leq \mathbb{P}\left(\exists f \in \mathcal{R}(\epsilon) \text{ s.t. } \hat{\ell}(f) > \epsilon + \varepsilon_n \right)  \text{ because } \lambda_{sup} -\lambda(f)  \geq 0 \text{ by definition of } \lambda_{sup}\\
    &= \mathbb{P}\left(\exists f \in \mathcal{R}(\epsilon) \text{ s.t. } \hat{\ell}(f) - \ell(f) > \epsilon + \varepsilon_n - \ell(f) \right) \text{ by subtracting $\ell(f)$ from both sides}  \\
    &\leq \mathbb{P}\left(\exists f \in \mathcal{R}(\epsilon) \text{ s.t. } \hat{\ell}(f) - \ell(f) > \varepsilon_n \right) \text{ because } \epsilon - \ell(f) \geq 0 \text{ for } f \in \mathcal{R}(\epsilon)\\
    &\leq \sum_{f \in \mathcal{R}(\epsilon)}\mathbb{P}\left( \hat{\ell}(f) - \ell(f)  > \varepsilon_n \right) \text{ by the Union bound } \\
    &\leq  \sum_{f \in \mathcal{R}(\epsilon)}  \exp\left\{ \frac{-2n\varepsilon_n^2}{(b-a)^2} \right\} \text{ by Hoeffding's inequality}\\
    &= |\mathcal{R}(\epsilon)| \exp\left\{ \frac{-2n\varepsilon_n^2}{(\ell_{\text{max}}-\ell_{\text{min}})^2} \right\}\\
    &\leq C \exp\left\{ \frac{-2n\varepsilon_n^2}{(\ell_{\text{max}}-\ell_{\text{min}})^2} \right\}
\end{align*}
for any $C \geq |\mathcal{R}(\epsilon)|$. 

We can form the complementary bound as follows

\begin{align*}
    &\mathbb{P}\left(\exists f \notin \mathcal{R}(\epsilon; \mathcal{F}, \ell, \lambda) \text{ s.t. } f \in \estrset(\epsilon + \varepsilon_n + \lambda_{sup}; \mathcal{F}, \ell, \lambda) \right)\\ 
    &= \mathbb{P}\left(\exists f \notin \mathcal{R}(\epsilon) \text{ s.t. } \hat{\ell}(f) + \lambda(f) \leq \epsilon + \varepsilon_n + \lambda_{sup} \right) \text{ by definition of Rashomon set exclusion}  \\
    &\geq \mathbb{P}\left(\exists f \notin \mathcal{R}(\epsilon) \text{ s.t. } \hat{\ell}(f) \leq \epsilon + \varepsilon_n \right)  \text{ because }  \lambda_{sup} - \lambda(f)  \geq 0 \text{ by definition of } \lambda_{sup}\\
    &= \mathbb{P}\left(\exists f \notin \mathcal{R}(\epsilon) \text{ s.t. } \hat{\ell}(f) - \ell(f) \leq \epsilon + \varepsilon_n - \ell(f) \right) \text{ by subtracting $\ell(f)$ from both sides}  \\
    &\geq \mathbb{P}\left(\exists f \notin \mathcal{R}(\epsilon) \text{ s.t. } \hat{\ell}(f) - \ell(f) \leq \varepsilon_n \right) \text{ because } \epsilon - \ell(f) \geq 0 \text{ for } f \notin \mathcal{R}(\epsilon)\\
    &= 1 - \mathbb{P}\left(\forall f \notin \mathcal{R}(\epsilon) \text{ s.t. } \hat{\ell}(f) - \ell(f) > \varepsilon_n \right)\\
    &\geq  1 - \sum_{f \in \mathcal{F}} \mathbb{P}\left( \hat{\ell}(f) - \ell(f)  > \varepsilon_n \right) \text{ by the Union bound } \\
    &\geq  1 - \sum_{f \in \mathcal{F}(\epsilon)}  \exp\left\{ \frac{-2n\varepsilon_n^2}{(b-a)^2} \right\} \text{ by Hoeffding's inequality}\\
    &= 1 - |\mathcal{F}| \exp\left\{ \frac{-2n\varepsilon_n^2}{(\ell_{\text{max}}-\ell_{\text{min}})^2} \right\}.
\end{align*}

Thus, for a given $\varepsilon_n$, we have that the probability of a type-1 error is \textbf{upper} bounded by $C \exp\left\{ \frac{-2n\varepsilon_n^2}{(\ell_{\text{max}}-\ell_{\text{min}})^2}\right\}$, and the probability of a type-2 error is \textbf{lower} bounded by $1 - |\mathcal{F}| \exp\left\{ \frac{-2n\varepsilon_n^2}{(\ell_{\text{max}}-\ell_{\text{min}})^2} \right\}$. 

Define $\delta:=C \exp\left\{ \frac{-2n\varepsilon_n^2}{(\ell_{\text{max}}-\ell_{\text{min}})^2}\right\}.$ We can rearrange the upper bound to express $\varepsilon_n$ in terms of a desired probability $\delta$ as follows:
\begin{align*}
    &C \exp\left\{ \frac{-2n\varepsilon_n^2}{(\ell_{\text{max}}-\ell_{\text{min}})^2}\right\} = \delta\\
    &\iff  \frac{-2n\varepsilon_n^2}{(\ell_{\text{max}}-\ell_{\text{min}})^2} = \ln(\frac{\delta}{C})\\
    &\iff  \varepsilon_n^2 = \frac{(\ell_{\text{max}}-\ell_{\text{min}})^2\ln(\frac{\delta}{C})}{-2n}\\
    &\iff  \varepsilon_n = \sqrt{\frac{(\ell_{\text{max}}-\ell_{\text{min}})^2\ln(\frac{\delta}{C})}{-2n}}\\
    &\iff  \varepsilon_n = \sqrt{\frac{(\ell_{\text{max}}-\ell_{\text{min}})^2\ln(\frac{C}{\delta})}{2n}}
\end{align*}
Thus, using this value of $\varepsilon_n,$ we have
\begin{align*}
    &\varepsilon_n = \sqrt{\frac{(\ell_{\text{max}}-\ell_{\text{min}})^2\ln(\frac{C}{\delta})}{2n}}\\
    &\implies \mathbb{P}\left(\poprset(\epsilon;\mathcal{F}, \ell,\lambda) \not\subseteq \estrset\left(\smarteps + \epsilon + \lambda_{\sup};\mathcal{F}, \ell, \lambda\right)\right) \leq \delta\\
    & \iff \mathbb{P}\left(\poprset(\epsilon;\mathcal{F}, \ell,\lambda) \subseteq \estrset\left(\smarteps + \epsilon + \lambda_{\sup};\mathcal{F}, \ell, \lambda\right)\right) \geq 1 - \delta
\end{align*}
as required.
\end{proof}

\begin{corollary}[Capturing Each Unobserved Submodel]
\label{app-cor:s_coverage}
    Let $\maxcondloss$ be defined as in Definition \ref{assm:bounded_loss}, and assume that $\allcondmodels \subseteq \mathcal{F}$. For any loss function $\ell$ bounded between $\ell_{\min}$ and $\ell_{\max}$, it holds that
 \begin{align*}
     \mathbb{P}\left(\allcondmodels \subseteq \estrset(\smarteps + \maxcondloss + \lambda_{\sup}; \mathcal{F}, \ell, \lambda)\right) \geq 1 - \delta,
 \end{align*}
 for a sample of size $n$ and 
 regularization penalty $\lambda$, with $\smarteps$ defined as in Theorem \ref{thm:smart_eps}.
\end{corollary}
\begin{proof}
We have that:
\begin{align*}
    &\mathbb{P}\left(\allcondmodels \subseteq \estrset(\epsilon + \varepsilon_n + \lambda_{sup}; \mathcal{F}, \ell, \lambda) \right)\\
    &=1 - \mathbb{P}\left(\exists f \in \allcondmodels \text{ s.t. } f \notin \estrset(\epsilon + \varepsilon_n + \lambda_{sup}; \mathcal{F}, \ell, \lambda) \right)\\
    &\geq 
    1 - \mathbb{P}\left(\exists f \in \mathcal{R}(\epsilon; \mathcal{F}, \ell, \lambda) \notin \estrset(\epsilon + \varepsilon_n + \lambda_{sup}; \mathcal{F}, \ell, \lambda) \right) \text{ because } \ell(f) \leq \epsilon \forall f \in \allcondmodels\\
    &= \mathbb{P}\left(\mathcal{R}(\epsilon; \mathcal{F}, \ell, \lambda) \subseteq \estrset(\epsilon + \varepsilon_n + \lambda_{sup}; \mathcal{F}, \ell, \lambda) \right)\\
    &\geq 1 - \delta
\end{align*}
for $\smarteps =  \sqrt{\frac{(\ell_{\text{max}} - \ell_{\text{min}})^2\ln\left(\frac{C}{\delta} \right)}{2n}}$ by Theorem \ref{app-thm:smart_eps}.

\end{proof}

\newpage
\begin{proposition}
\label{prop:rset_size_mse}
The size of the estimated Rashomon set with threshold $\epsilon' > \epsilon$ is an upper bound on the size of the population Rashomon set with threshold $\epsilon$ with high probability:
\begin{align*}
    \mathbb{P}_{(\Obs, Y) \sim \ObsDist}\left( |\estrset(\epsilon')| > |\poprset(\epsilon)| \right) = 1 - O(n^{-1}).
\end{align*}
\end{proposition}

\begin{proof}
    Recall how we define our estimated and population Rashomon sets:
    \begin{align*}
    \poprset(\epsilon; \mathcal{F}, \ell, \lambda) &=\{f \in \mathcal{F} :
    \mathbb{E}_{(\Obs, Y) \sim \ObsDist}[\ell(f, \Obs, Y; \lambda) + \lambda(f)] \leq \epsilon \}
    \\
        \estrset(\epsilon'; \mathcal{F}, \ell, \lambda) &= \left\{f \in \mathcal{F} : \frac{1}{n} \sum_{i=1}^n\ell(f, \obs_i, y_i) + \lambda(f) \leq \epsilon' \right\}.
    \end{align*}
Before moving further, recognize that if all models in $\poprset(\epsilon)$ are members of the estimated Rashomon set $\estrset(\epsilon')$ (i.e., $\poprset(\epsilon) \subseteq \estrset(\epsilon')$, then $|\estrset(\epsilon')| > |\poprset(\epsilon)|$). Therefore, we will find a lower bound on the probability of $|\estrset(\epsilon')| > |\poprset(\epsilon)|$ and show that this goes to 1 asymptotically.
\begin{align*}
    &\mathbb{P}_{(\Obs, Y) \sim \ObsDist}\left( |\estrset(\epsilon')| > |\poprset(\epsilon)| \right) \\
    &\geq \mathbb{P}_{(\Obs, Y) \sim \ObsDist}\left(  \poprset(\epsilon)  \subseteq \estrset(\epsilon')  \right) \text{ (from above discussion)}\\
    &= \mathbb{P}_{(\Obs, Y) \sim \ObsDist}(\forall f \in \poprset(\epsilon), f \in \estrset(\epsilon')) \text{ (by definition of $\subseteq$)}\\
    &= 1 - \mathbb{P}_{(\Obs, Y) \sim \ObsDist}(\exists f \in \poprset(\epsilon), f \notin \estrset(\epsilon')) \text{ (by law of total probability)}\\
    &\geq 1 - \sum_{f \in \poprset(\epsilon)}\mathbb{P}_{(\Obs, Y) \sim \ObsDist}( f \notin \estrset(\epsilon') ) \text{ (by Union bound)}.
\end{align*}

Let us work on bounding this model-level probability. First, recall when a model is in the estimated and population Rashomon sets: when the empirical average or expectation over sample-level losses is below the specified threshold. So, we now want to bound for a given model $f$ in the population $\epsilon$-threshold Rashomon set,
\begin{align*}
    &\mathbb{P}_{(\Obs, Y) \sim \ObsDist}\left( f \notin \estrset(\epsilon') \right) \\
    &=\mathbb{P}_{(\Obs, Y) \sim \ObsDist}\left( \frac{1}{n}\sum_{i = 1}^n \ell(f, X_i, Y_i; \lambda) + \lambda(f) > \epsilon' \right) \\
&=\mathbb{P}_{(\Obs, Y) \sim \ObsDist}\left( \frac{1}{n}\sum_{i = 1}^n \ell(f, X_i, Y_i; \lambda) - \mathbb{E}_{(\Obs, Y) \sim \ObsDist}[\ell(f, X, Y; \lambda)] > \epsilon' - \mathbb{E}_{(\Obs, Y) \sim \ObsDist}[\ell(f, X, Y; \lambda)] \right) \\
&\leq \frac{\mathbb{E}_{(\Obs, Y) \sim \ObsDist}\left[\left( \frac{1}{n}\sum_{i = 1}^n \ell(f, X_i, Y_i; \lambda) - \mathbb{E}_{(\Obs, Y) \sim \ObsDist}[\ell(f, X, Y; \lambda)]\right)^2\right]}{\left(\epsilon' - \mathbb{E}_{(\Obs, Y) \sim \ObsDist}[\ell(f, X, Y; \lambda)] \right)^2} \text{ (by Chebyshev's inequality)} \\
&= \frac{\sum_{i = 1}^n\mathbb{E}_{(\Obs, Y) \sim \ObsDist}\left[\left( \ell(f, X_i, Y_i; \lambda) - \mathbb{E}_{(\Obs, Y) \sim \ObsDist}[\ell(f, X, Y; \lambda)]\right)^2\right]}{n^2 \left(\epsilon' - \mathbb{E}_{(\Obs, Y) \sim \ObsDist}[\ell(f, X, Y; \lambda)] \right)^2} \text{ (because data are IID)} \\
&\leq \frac{\sum_{i = 1}^n (\ell_{\max} - \ell_{\min})^2}{4n^2 \left(\epsilon' - \mathbb{E}_{(\Obs, Y) \sim \ObsDist}[\ell(f, X, Y; \lambda)] \right)^2}  \text{ (because range of loss is bounded)} \\
&= \frac{ (\ell_{\max} - \ell_{\min})^2}{4n \left(\epsilon' - \mathbb{E}_{(\Obs, Y) \sim \ObsDist}[\ell(f, X, Y; \lambda)] \right)^2}  \text{ (by canceling out constant factors)} \\
&\leq \frac{(\ell_{\max} - \ell_{\min})^2}{4n \left(\epsilon' - \epsilon \right)^2}  \text{ (because expected loss of $f$ is smaller than $\epsilon$)}.
\end{align*}

Plugging this last result back into the Union bound across all models, we can see then that
\begin{align*}
    \mathbb{P}_{(\Obs, Y) \sim \ObsDist}\left( |\estrset(\epsilon')| > |\poprset(\epsilon)| \right)
    &\geq 1 - \sum_{f \in \poprset(\epsilon)}  \frac{(\ell_{\max} - \ell_{\min})^2}{4n \left(\epsilon' - \epsilon \right)^2} \\
    &= 1 - |\poprset(\epsilon)|\frac{ (\ell_{\max} - \ell_{\min})^2}{4n \left(\epsilon' - \epsilon \right)^2} \\ 
    &= 1 - O(n^{-1}) \text{ (because all other factors are constants)}
\end{align*}

\end{proof} 

\begin{corollary} \label{app-cor:asymp_rset_conv}
    Let $\maxcondloss$ be defined as in Definition \ref{assm:bounded_loss}, and assume that $\allcondmodels \subseteq \mathcal{F}$. Define 
    \begin{align*}
         \estsmarteps = \sqrt{\frac{(\ell_{\max}-\ell_{\min})^2\ln\left(\frac{|\estrset(\epsilon' + \lambda_{\sup})|}{\delta} \right)}{2n}},
        \end{align*}
    using the size of an estimated Rashomon set with threshold $\epsilon' > \maxcondloss$. Then, 
    for any loss function $\ell$ bounded between $\ell_{\min}$ and $\ell_{\max}$, it holds that 
 \begin{align*}
     \mathbb{P}\left(\allcondmodels \subset \estrset(\estsmarteps + \maxcondloss + \lambda_{\sup}; \mathcal{F}, \ell, \lambda)\right) \geq 1 - \delta + O(n^{-1}),
 \end{align*}
 given a sample of size $n$ and 
 regularization penalty $\lambda$.
\end{corollary}
\begin{proof}
    The proof connects Corollary~\ref{cor:s_coverage} with Proposition~\ref{prop:rset_size_mse}:
    \begin{align*}
         &\mathbb{P}\left(\allcondmodels \subset \estrset(\estsmarteps + \maxcondloss + \lambda_{\sup}; \mathcal{F}, \ell, \lambda)\right) \\
         &\geq  \mathbb{P}\left(\allcondmodels \subset \estrset(\estsmarteps + \maxcondloss + \lambda_{\sup}; \mathcal{F}, \ell, \lambda) \hspace{2px} \big| \hspace{2px} |\estrset(\epsilon' + \lambda_{\sup})| > |\poprset(\maxcondloss)| \right)\mathbb{P}\left(|\estrset(\epsilon' + \lambda_{\sup})| > |\poprset(\maxcondloss)| \right) \\
         &\geq (1 - \delta)\left(1 -  O(n^{-1})\right) \text{ (from Theorem~\ref{thm:smart_eps} and Proposition~\ref{prop:rset_size_mse})}\\
         &= 1 - \delta - (1 - \delta)O(n^{-1}) \\
         &= 1 - \delta - O(n^{-1}).
    \end{align*}
\end{proof}

\newpage

\begin{theorem}[Capturing Variable Importance for Each Submodel]
\label{app-thm:vi_coverage_uncond}
    Let $\gamma \in (0, 1)$ and $\vibound$ be a value such that, 
    $$\mathbb{P}_{\mathcal{D}^{(n)}}\left(\forall f \in \allcondmodels, \Phi_j(f, \ObsDist) \in \left[ \Phi_j(f, \D) \pm \vibound \right]\right) \geq 1 - \gamma,$$
    where $\mathbb{P}_{\mathcal{D}^{(n)}}$ denotes the probability of drawing the observed $n$ samples from $\mathcal{P}_{XY}$. It follows that, for all $f \in \allcondmodels,$ with probability at least $1 - (\delta + \gamma),$
    \begin{align*}
     \left\{\Phi_j(f, \ObsDist) \mid \condsubmodelNoInput{} \in \allcondmodels \right\} \subseteq
    \begin{bmatrix} 
     &\inf_{f' \in \estrset(\smarteps + \maxcondloss + \lambda_{\sup})}\Phi_j(f', \D) - \vibound, \\
     &\sup_{f' \in \estrset(\smarteps + \maxcondloss + \lambda_{\sup})}\Phi_j(f', \D) + \vibound
     \end{bmatrix},
    \end{align*}
    where $\smarteps$ and $\delta$ are defined as in Corollary \ref{cor:s_coverage}. 
\end{theorem}

\begin{proof}
    \begin{align*}
        &\mathbb{P}\left(\left\{\Phi_j(f, \ObsDist) \mid \condsubmodelNoInput{} \in \allcondmodels \right\} \subseteq \left[ \inf_{f' \in \estrset(\smarteps + \maxcondloss + \lambda_{\sup})} \Phi_j(f', \D) - \alpha, \sup_{f' \in \estrset(\smarteps + \maxcondloss + \lambda_{\sup})} \Phi_j(f', \D) +\alpha \right]\right)\\
        &\geq \mathbb{P}\bigg(
            \left( 
                \allcondmodels \subseteq \estrset(\smarteps + \maxcondloss + \lambda_{\sup})
            \right)\cap 
            \left( 
                \Phi_j(\condsubmodelNoInput{u}, \ObsDist) \in \left[ \Phi_j(\condsubmodelNoInput{u}, \D) - \alpha,  \Phi_j(\condsubmodelNoInput{u}, \D) +\alpha \right] \forall \condsubmodelNoInput{u} \in \allcondmodels
            \right)
        \bigg)\\
        &\text{ Because $\allcondmodels \subseteq \estrset$ and the true VI for each $f_u$ contained in the interval is a sufficient condition for above}\\
        &= 1 - \mathbb{P}\left(
            \left( 
                \allcondmodels \not\subseteq \estrset(\smarteps + \maxcondloss + \lambda_{\sup})
            \right) \cup 
            \left(  \exists \condsubmodelNoInput{u} \in \allcondmodels \text{ s.t. }
                \Phi_j(\condsubmodelNoInput{u}, \ObsDist) \notin \left[ \Phi_j(\condsubmodelNoInput{u}, \D) - \alpha,  \Phi_j(\condsubmodelNoInput{u}, \D) +\alpha \right]
            \right)
        \right)\\
        &\text{By the law of total probability} \\
        &\geq 1 - \mathbb{P}\left(
            \allcondmodels \not\subseteq \estrset(\smarteps + \maxcondloss + \lambda_{\sup})
        \right) -
        \mathbb{P} \left( \exists \condsubmodelNoInput{u} \in \allcondmodels \text{ s.t. } 
                \Phi_j(\condsubmodelNoInput{u}, \ObsDist) \notin \left[ \Phi_j(\condsubmodelNoInput{u}, \D) - \alpha,  \Phi_j(\condsubmodelNoInput{u}, \D) +\alpha \right]
            \right)\\
        &\text{ By the Union Bound}\\
        &= \mathbb{P}\left(
            \allcondmodels \subseteq \estrset(\smarteps + \maxcondloss + \lambda_{\sup})
        \right) -
        \mathbb{P} \left( \exists \condsubmodelNoInput{u} \in \allcondmodels \text{ s.t. } 
                \Phi_j(\condsubmodelNoInput{u}, \ObsDist) \notin \left[ \Phi_j(\condsubmodelNoInput{u}, \D) - \alpha,  \Phi_j(\condsubmodelNoInput{u}, \D) + \alpha \right]
            \right)\\
        &\geq 1 - \delta -
        \mathbb{P} \left( \exists \condsubmodelNoInput{u} \in \allcondmodels \text{ s.t. }
                \Phi_j(\condsubmodelNoInput{u}, \ObsDist) \notin \left[ \Phi_j(\condsubmodelNoInput{u}, \D) - \alpha,  \Phi_j(\condsubmodelNoInput{u}, \D) + \alpha \right]
            \right) \text{ By Corollary \ref{app-cor:s_coverage}}\\
        &\geq 1 - \delta -
        \gamma \text{ By definition of } \alpha
    \end{align*}

That is,
\begin{align*}
    &\mathbb{P}\left(\left\{\Phi_j(f, \ObsDist) \mid \condsubmodelNoInput{} \in \allcondmodels \right\} \subseteq \left[ \sup_{f' \in \estrset(\smarteps + \maxcondloss + \lambda_{\sup})} \Phi_j(f', \D) - \alpha, \inf_{f' \in \estrset(\smarteps + \maxcondloss + \lambda_{\sup})} \Phi_j(f', \D) +\alpha \right]\right)\\
    &\geq 1 - (\delta +
        \gamma)
\end{align*}
as required.

\end{proof}

\newpage

\begin{lemma}
\label{app-lem:optmodel_coverage}
Let $\maxdistshift_j$ be defined as in Assumption~\ref{assm:l1_dist}. It holds that
$$
    \Phi_j(\optmodel, \UnobsDist) \in \left[\min_{\condsubmodelNoInput{u}}\Phi_j(\condsubmodelNoInput{u}, \ObsDist) - \maxdistshift_j, \max_{\condsubmodelNoInput{u}}\Phi_j(\condsubmodelNoInput{u}, \ObsDist) + \maxdistshift_j\right]
$$
    
\end{lemma}

\begin{proof}
    First, realize that the variable importance for $\optmodel$ on the complete data distribution--which includes $X, U, Y$--is actually a convex combination of the variable importance for each $\condsubmodelNoInput{}$ on its corresponding conditional data distribution:
    \begin{align*}
        \Phi_j(\optmodel, \UnobsDist) &= \mathbb{E}_{(\Obs, \Unobs, Y) \sim \UnobsDist}\left[ \phi_j(\optmodel, (\Obs, \Unobs, Y)) \right] \text{ by definition of $\Phi_j$} \\
        &= \sum_{u \in \mathcal{U}}\mathbb{P}(\Unobs = u)\mathbb{E}_{(X, Y) \sim \CondDist{u}}\left[ \phi_j(\optmodel, (\Obs, u, Y)) \mid \Unobs = u \right] \text{ by the law of iterated expectation}.
    \end{align*}
    Now, because we have conditioned on $U = u$, we know that $\optmodel = \condsubmodelNoInput{u}.$ This equality means that we can rewrite 
    \begin{align*}
        &\sum_{u \in \mathcal{U}}\mathbb{P}(\Unobs = u)\mathbb{E}_{(X, Y) \sim \CondDist{u}}\left[ \phi_j(\optmodel, (\Obs, u, Y)) \mid \Unobs = u \right] \\
        = &\sum_{u \in \mathcal{U}}\mathbb{P}(\Unobs = u)\mathbb{E}_{(X, Y) \sim \CondDist{u}}\left[ \phi_j(\condsubmodelNoInput{u}, (\Obs, Y)) \mid \Unobs = u \right] \\
        = &\sum_{u \in \mathcal{U}}\mathbb{P}(\Unobs = u)\Phi_j(\condsubmodelNoInput{u}, \CondDist{u}) \text{ by definition of $\Phi_j$}.
    \end{align*}
    In other words, $\Phi_j(\optmodel, \UnobsDist)$ is a convex combination of $\Phi_j(\condsubmodelNoInput{u}, \CondDist{u})$ across all $u \in \mathcal{U}.$ Because of this convexity, we then know that $\Phi_j(\optmodel, \UnobsDist)$ is upper and lower bounded by the conditional submodel importances:
    \begin{align*}
        \min_u \Phi_j(\condsubmodelNoInput{u}, \CondDist{u}) \leq \Phi_j(\optmodel, \UnobsDist) \leq \max_u \Phi_j(\condsubmodelNoInput{u}, \CondDist{u}).
    \end{align*}
    Unfortunately, these upper and lower bounds are still in terms of non-identifiable distributions because we never know when $U = u$ or $U \neq u.$ However, we will next show how we can construct upper and lower bounds on $\Phi_j(\condsubmodelNoInput{u}, \CondDist{u})$ in terms of the observable distribution $\ObsDist$ and our parameter $\maxdistshift_j.$

    Choose some $u \in \mathcal{U}.$ Recognize that we can compute the bias in our variable importance estimate by examining the difference between $\Phi_j(\condsubmodelNoInput{u}, \CondDist{u})$ and $\Phi_j(\condsubmodelNoInput{u}, \ObsDist).$ First, recognize by the law of iterated expectation that
    \begin{align*}
        \Phi_j(\condsubmodelNoInput{u}, \ObsDist) = \mathbb{P}(U = u)\Phi_j(\condsubmodelNoInput{u}, \CondDist{u}) + \mathbb{P}(U \neq u)\Phi_j(\condsubmodelNoInput{u}, \mathcal{P}_{XY|U \neq u}).
    \end{align*}
    So, we can rewrite the difference between conditional and marginal VI for $\condsubmodelNoInput{u}$ as
    \begin{align*}
        &\Phi_j(\condsubmodelNoInput{u}, \CondDist{u}) - \Phi_j(\condsubmodelNoInput{u}, \ObsDist) \\
        &= \Phi_j(\condsubmodelNoInput{u}, \CondDist{u}) - \mathbb{P}(U = u)\Phi_j(\condsubmodelNoInput{u}, \CondDist{u}) - \mathbb{P}(U \neq u)\Phi_j(\condsubmodelNoInput{u}, \mathcal{P}_{XY|U \neq u})) \\
        &= \left(1  - \mathbb{P}(U = u) \right)\Phi_j(\condsubmodelNoInput{u}, \CondDist{u}) - \mathbb{P}(U \neq u)\Phi_j(\condsubmodelNoInput{u}, \mathcal{P}_{XY|U \neq u})) \\
        &=  \mathbb{P}(U \neq u)\left[\Phi_j(\condsubmodelNoInput{u}, \CondDist{u}) - \Phi_j(\condsubmodelNoInput{u}, \mathcal{P}_{XY|U \neq u}))\right]
    \end{align*}
    where the last two lines come the law of total probability and factoring out the $\mathbb{P}(U \neq u)$ term. In other words, the distance between the conditional VI and marginal VI for $\condsubmodelNoInput{u}$ is broken into two parts: what are the chances that an observation is in another unobserved subgroup, and what is the distance between the VI for conditional subgroups. Let us use this decomposition to first find an upper bound on $\Phi_j(\condsubmodelNoInput{u}, \CondDist{u})$:
    \begin{align*}
        \Phi_j(\condsubmodelNoInput{u}, \CondDist{u})  &= \Phi_j(\condsubmodelNoInput{u}, \ObsDist) + \mathbb{P}(U \neq u)\left[\Phi_j(\condsubmodelNoInput{u}, \CondDist{u}) - \Phi_j(\condsubmodelNoInput{u}, \mathcal{P}_{XY|U \neq u}))\right] \\
        &\leq \Phi_j(\condsubmodelNoInput{u}, \ObsDist) + \mathbb{P}(U \neq u)\left|\Phi_j(\condsubmodelNoInput{u}, \CondDist{u}) - \Phi_j(\condsubmodelNoInput{u}, \mathcal{P}_{XY|U \neq u}))\right| \\
        &\qquad \text{by definition of absolute value} \\
        &\leq \Phi_j(\condsubmodelNoInput{u}, \ObsDist) + \left|\Phi_j(\condsubmodelNoInput{u}, \CondDist{u}) - \Phi_j(\condsubmodelNoInput{u}, \mathcal{P}_{XY|U \neq u}))\right| \\
        &\qquad \text{ because $\mathbb{P}(U \neq u) \leq 1$} \\
        &\leq \Phi_j(\condsubmodelNoInput{u}, \ObsDist) + \max_{u \in \mathcal{U}}\left|\Phi_j(\condsubmodelNoInput{u}, \CondDist{u}) - \Phi_j(\condsubmodelNoInput{u}, \mathcal{P}_{XY|U \neq u}))\right| \\
        &\qquad \text{ by definition of $\max$} \\
        &= \Phi_j(\condsubmodelNoInput{u}, \ObsDist) + \maxdistshift_j \text{ by definition of $\maxdistshift_j$} \\
        &\leq \max_{u \in \mathcal{U}} \Phi_j(\condsubmodelNoInput{u}, \ObsDist) + \maxdistshift_j \text{ by definition of $\max$}
    \end{align*}
    Our upper bound shows that for any conditional submodel, the VI for its corresponding conditional subgroup $\CondDist{u}$ is bounded by the largest conditional submodel VI applied to the observed distribution $\ObsDist$ plus some correction factor $\maxdistshift_j$. This correction factor measures how different is the VI for any conditional submodel between different subgroups.

    Now, we will derive a lower bound. These steps are very similar to the steps for deriving the upper bound are included for completeness:
    \begin{align*}
        \Phi_j(\condsubmodelNoInput{u}, \CondDist{u})  &= \Phi_j(\condsubmodelNoInput{u}, \ObsDist) + \mathbb{P}(U \neq u)\left[\Phi_j(\condsubmodelNoInput{u}, \CondDist{u}) - \Phi_j(\condsubmodelNoInput{u}, \mathcal{P}_{XY|U \neq u}))\right] \\
        &\geq \Phi_j(\condsubmodelNoInput{u}, \ObsDist) - \mathbb{P}(U \neq u)\left|\Phi_j(\condsubmodelNoInput{u}, \CondDist{u}) - \Phi_j(\condsubmodelNoInput{u}, \mathcal{P}_{XY|U \neq u}))\right| \\
        &\qquad \text{by definition of absolute value} \\
        &\geq \Phi_j(\condsubmodelNoInput{u}, \ObsDist) - \left|\Phi_j(\condsubmodelNoInput{u}, \CondDist{u}) - \Phi_j(\condsubmodelNoInput{u}, \mathcal{P}_{XY|U \neq u}))\right| \\
        &\qquad \text{ because $\mathbb{P}(U \neq u) \geq 0$} \\
        &\geq \Phi_j(\condsubmodelNoInput{u}, \ObsDist) - \max_{u \in \mathcal{U}}\left|\Phi_j(\condsubmodelNoInput{u}, \CondDist{u}) - \Phi_j(\condsubmodelNoInput{u}, \mathcal{P}_{XY|U \neq u}))\right| \\
        &\qquad \text{ by definition of $-\max$} \\
        &= \Phi_j(\condsubmodelNoInput{u}, \ObsDist) - \maxdistshift_j \text{ by definition of $\maxdistshift_j$} \\
        &\geq \min_{u \in \mathcal{U}} \Phi_j(\condsubmodelNoInput{u}, \ObsDist) - \maxdistshift_j \text{ by definition of $\min$}.
    \end{align*}
    Now, we will use these upper and lower bounds for any conditional submodel to bound $\Phi_j(\optmodel, \UnobsDist).$ 

    We will first focus on the lower bound. From the convex combination argument from earlier, we know that
    \begin{align*}
        \Phi_j(\optmodel, \UnobsDist) \geq \min_{u \in \mathcal{U}} \Phi_j(\condsubmodelNoInput{u}, \CondDist{u}).
    \end{align*}
    From our derived lower bound at the conditional submodel level, we know that 
    \begin{align*}
        \min_{u \in \mathcal{U}} \Phi_j(\condsubmodelNoInput{u}, \CondDist{u}) \geq \min_{u \in \mathcal{U}} \Phi_j(\condsubmodelNoInput{u}, \ObsDist) - \tau_j.
    \end{align*}
    Therefore,
    \begin{align*}
        \Phi_j(\optmodel, \UnobsDist) \geq \min_{u \in \mathcal{U}} \Phi_j(\condsubmodelNoInput{u}, \ObsDist) - \tau_j.
    \end{align*}
    Because we similarly know that
    \begin{align*}
        \max_{u \in \mathcal{U}} \Phi_j(\condsubmodelNoInput{u}, \CondDist{u}) \leq \max_{u \in \mathcal{U}} \Phi_j(\condsubmodelNoInput{u}, \ObsDist) + \tau_j,
    \end{align*}
    we can also conclude that
    \begin{align*}
        \Phi_j(\optmodel, \UnobsDist) \leq \min_{u \in \mathcal{U}} \Phi_j(\condsubmodelNoInput{u}, \ObsDist) + \tau_j.
    \end{align*}
    That is, 
$$
    \Phi_j(\optmodel, \UnobsDist) \in \left[\min_{\condsubmodelNoInput{u}}\Phi_j(\condsubmodelNoInput{u}, \ObsDist) - \maxdistshift_j, \max_{\condsubmodelNoInput{u}}\Phi_j(\condsubmodelNoInput{u}, \ObsDist) + \maxdistshift_j\right]
$$ 
as required.
\end{proof}

\newpage

\begin{theorem}
\label{app-thm:optmodel_coverage}
Let $\vibound$ and $\gamma$ be defined as in Theorem \ref{thm:vi_coverage_uncond}, $\smarteps$ and $\delta$ be defined as in Theorem~\ref{cor:s_coverage}, and $\maxdistshift_j$ be defined as in Assumption~\ref{assm:l1_dist}.

With probability at least $1 - (\delta + \gamma), $
    $$
    \Phi_j(\optmodel, \UnobsDist)\in\left[\inf_{f \in \estrset(\smarteps + \maxcondloss + \lambda_{\sup})}\Phi_j(f, \D) - \maxdistshift_j - \vibound, 
     \sup_{f \in \estrset(\smarteps + \maxcondloss + \lambda_{\sup})}\Phi_j(f, \D) + \maxdistshift_j + \vibound
     \right]
    $$
\end{theorem}

\begin{proof}
Theorem \ref{app-thm:vi_coverage_uncond} states that, with probability at least $1 - (\delta + \gamma),$
\begin{align*}
     &\left\{\Phi_j(f, \ObsDist) \mid \condsubmodelNoInput{} \in \allcondmodels \right\} 
     \subseteq \left[ \inf_{f' \in \estrset(\smarteps + \maxcondloss + \lambda_{\sup})} \Phi_j(f', \D) - \vibound, \sup_{f' \in \estrset(\smarteps + \maxcondloss + \lambda_{\sup})} \Phi_j(f', \D) + \vibound \right].
\end{align*}
This implies that
\begin{align*}
    \inf_{f' \in \estrset(\smarteps + \maxcondloss + \lambda_{\sup})}\Phi_j(f', \D) - \vibound &\leq \inf_{\condsubmodelNoInput{u} \in \allcondmodels}\Phi_j(\condsubmodelNoInput{u}, \ObsDist)\\
    \iff \inf_{f' \in \estrset(\smarteps + \maxcondloss + \lambda_{\sup})}\Phi_j(f', \D) - \vibound - \maxdistshift_j &\leq \inf_{\condsubmodelNoInput{u} \in \allcondmodels}\Phi_j(\condsubmodelNoInput{u}, \ObsDist) - \maxdistshift_j\\
    &\leq \Phi_j(\optmodel, \UnobsDist) & \text{By Lemma \ref{app-lem:optmodel_coverage}}
\end{align*}
A symmetric argument applies for the upper bound, yielding that, with probability at least $1 - (\delta + \gamma),$
    $$
     \Phi_j(\optmodel, \UnobsDist)\in\left[\inf_{f \in \estrset(\smarteps + \maxcondloss + \lambda_{\sup})}\Phi_j(f, \D) - \maxdistshift_j - \vibound, 
     \sup_{f \in \estrset(\smarteps + \maxcondloss + \lambda_{\sup})}\Phi_j(f, \D) + \maxdistshift_j + \vibound
     \right]
    $$
as required.

\end{proof}

\newpage

\newpage
\section{EXPERIMENTAL DETAILS}
\label{app:experimental_details}
All experiments were run using TreeFarms \cite{xin2022exploring} to compute the Rashomon set of decision trees, with a depth bound of 3 and a regularization value of $0.001$. For computational efficiency, we dropped all but the first 8 variables (before binarization) from each dataset in our semi-synthetic experiments (except for the Compas dataset, which only has 7 features); during binarization, each variable was processed into 3 binary variables as evenly spaced quantiles over the distribution of the original input variable. We used 80\% of all data in each setting to compute the Rashomon set, and the remaining 20\% to estimate variable importance over this set. 

Note that wine quality \cite{cortez2009modeling} reports the rating for each sample on a 0 to 10 scale. We converted this into  binary classification problem where the goal was to predict whether the rating was greater than or equal to 5.

All experiments for this work were performed on an academic institution’s cluster computer. We used up to 10 machines in parallel, each with a Dell R730’s with 2 Intel Xeon E5-2640 Processors (40 cores).

\newpage
\section{EXPERIMENTS WITH APPROXIMATE RANDOM FOREST RASHOMON SETS} 
\label{app:random_forest_results}
The experimental results in the main body of this work used the Rashomon set of decision trees, calculated using TreeFarms \citep{xin2022exploring}. However, the \ourmethodAcronym{} framework directly applies to any model class for which the Rashomon set can be computed or well-approximated. In this section, we consider the approximate bounds on variable importance across the random forest Rashomon set introduced by \citet{laberge2023partial}. This approach does not explicitly enumerate the Rashomon set, but allows us to estimate $\inf_{f \in \estrset(\smarteps + \maxcondloss + \lambda_{\sup})}\Phi_j(f, \D)$ and $\sup_{f \in \estrset(\smarteps + \maxcondloss + \lambda_{\sup})}\Phi_j(f, \D),$ which are all that we require for \ourmethodAcronym.

We repeat each of the semi-synthetic experiments from Section \ref{sec:experiments} using these bounds. 
Given a tabular dataset $\mathcal{D}^{(n)}$, we first split the dataset into $K$ equally sized partitions $\mathcal{D}^{(n /K)}_1, \mathcal{D}^{(n /K)}_2, \hdots \mathcal{D}^{(n /K)}_K$. 
We then fit a decision tree classifier on each partition sequentially, yielding $K$ distinct models $\condsubmodelNoInput{1}, \condsubmodelNoInput{2}, \hdots \condsubmodelNoInput{K}$. 
For this set of experiments, we increase the complexity of the true data generation process by fitting full depth 6 decision trees.
In order to ensure that these models are distinct, when fitting the $k$-th tree, we set all entries of each feature used by the previously fitted models $\condsubmodelNoInput{1}, \condsubmodelNoInput{2}, \hdots \condsubmodelNoInput{k-1}$
equal to 0. 
For each partition, we create a semi-synthetic dataset where each label is replaced by the prediction of the corresponding model, i.e., $\bar{\mathcal{D}}^{(n/K)}_k := \{(X_i, \condsubmodel{k}{X_i})\}^{k n/K}_{i=(k-1)n/K}$.
Finally, we combine the semi-synthetic datasets back into one.
In doing so, we treat the partition to which each sample was assigned as an important unobserved feature $U \in \mathcal{U} := \{1, 2, \hdots, K\}$; if the $i$-th sample was assigned to partition $k$, we say $u_i = k$.
This yields a known true conditional mean function $\optmodel(x_i, u_i) := \condsubmodel{u_i}{x_i}, u_i \in \{1, 2, \hdots, K\}$. 
This setup allows us to know each ground truth quantity of interest ($\optmodel$ and $\allcondmodels$ with their corresponding variable importance values, $\maxcondloss$) while working with a fairly realistic distribution for $X$. 

We again apply this procedure to four disparate real-world datasets with $K=2$:
\begin{enumerate}
    \item Compas \citep{propublica}, which concerns criminal recidivism prediction for 6,907 individuals;
    \item Dropout \citep{martins2021early}, which concerns predicting whether students will drop out of college for 4,424 individuals;
    \item FICO \citep{competition}, which concerns predicting whether an individual will repay line of credit within 2 years over 10,459 individuals;
    \item Wine Quality \citep{cortez2009modeling}, which concerns predicting whether a wine will be highly rated over 6,497 wines.
\end{enumerate}

For each dataset, we evaluate our framework using random draws with replacement of sample sizes 100; 1,000; 10,000; 50,000; 250,000; and 500,000 from $\bar{\mathcal{D}}^{(n)}$, each over 100 iterations. 
For these experiments, we include every feature from each dataset.
In this section, we consider the model class $\mathcal{F}$ of random forests, and restrict each $\condsubmodelNoInput{k}$ to have a depth of at most $6$. To faithfully evaluate the theoretical claims of Section \ref{sec:methods}, we again use the true value for $\maxcondloss$ and $\maxdistshift_j$ in each experiment.


\begin{figure*}
    \centering
    \includegraphics[width=0.7\linewidth]{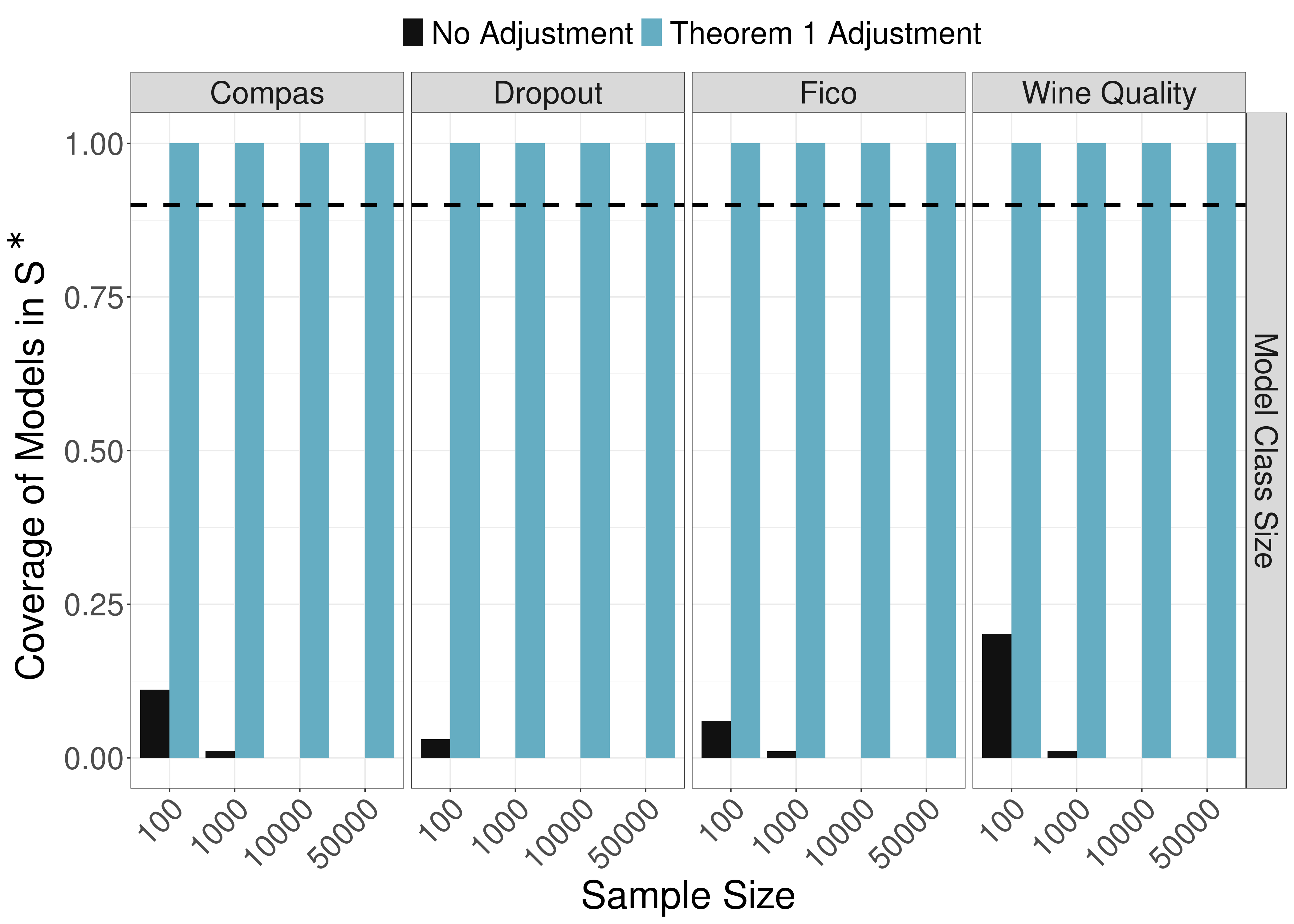}
    \caption{Verifying Theorem~\ref{thm:smart_eps} in finite sample datasets using random forests. We compute the proportion of 100 random draws of the each dataset in which Rashomon sets estimated with the Rashomon threshold adjusting for finite sample biases as in Theorem \ref{thm:smart_eps} (in blue) and without any adjustment (in black) captures each $\condsubmodelNoInput{u}$ for each setting. The target coverage rate is $\geq 0.9$, with $\delta=0.1$. Across all sample sizes and datasets, omitting finite sample adjustments yields Rashomon sets that leave out necessary models. In contrast, our adjustment yields the target coverage rate, verifying the theorem holds. } \label{fig:rf_submodel_capture_important_unobs}
\end{figure*}

\subsection{Finding each model in $\allcondmodels$} 

Corollary \ref{cor:s_coverage} demonstrates how we may select a value $\varepsilon_n$ such that, with high probability, the empirical Rashomon set contains $\allcondmodels$. We first evaluate the validity of this Corollary with a target confidence of $1 - \delta = 0.9$.

Figure \ref{fig:rf_submodel_capture_important_unobs} presents the results of this evaluation for the random forest Rashomon set. We again find that, across all six sample sizes and all four reference datasets, omitting the finite sample adjustment yields Rashomon sets that leave out necessary models. In contrast, \textbf{our empirical Rashomon set based on Corollary \ref{cor:s_coverage} contains every model in $\allcondmodels$ at above the specified rate}. These results demonstrate that adjusting for finite sample bias is necessary for constructing high-probability, finite-sample covers of the population Rashomon sets. 

\begin{figure*}
    \centering
    \includegraphics[width=0.7\linewidth]{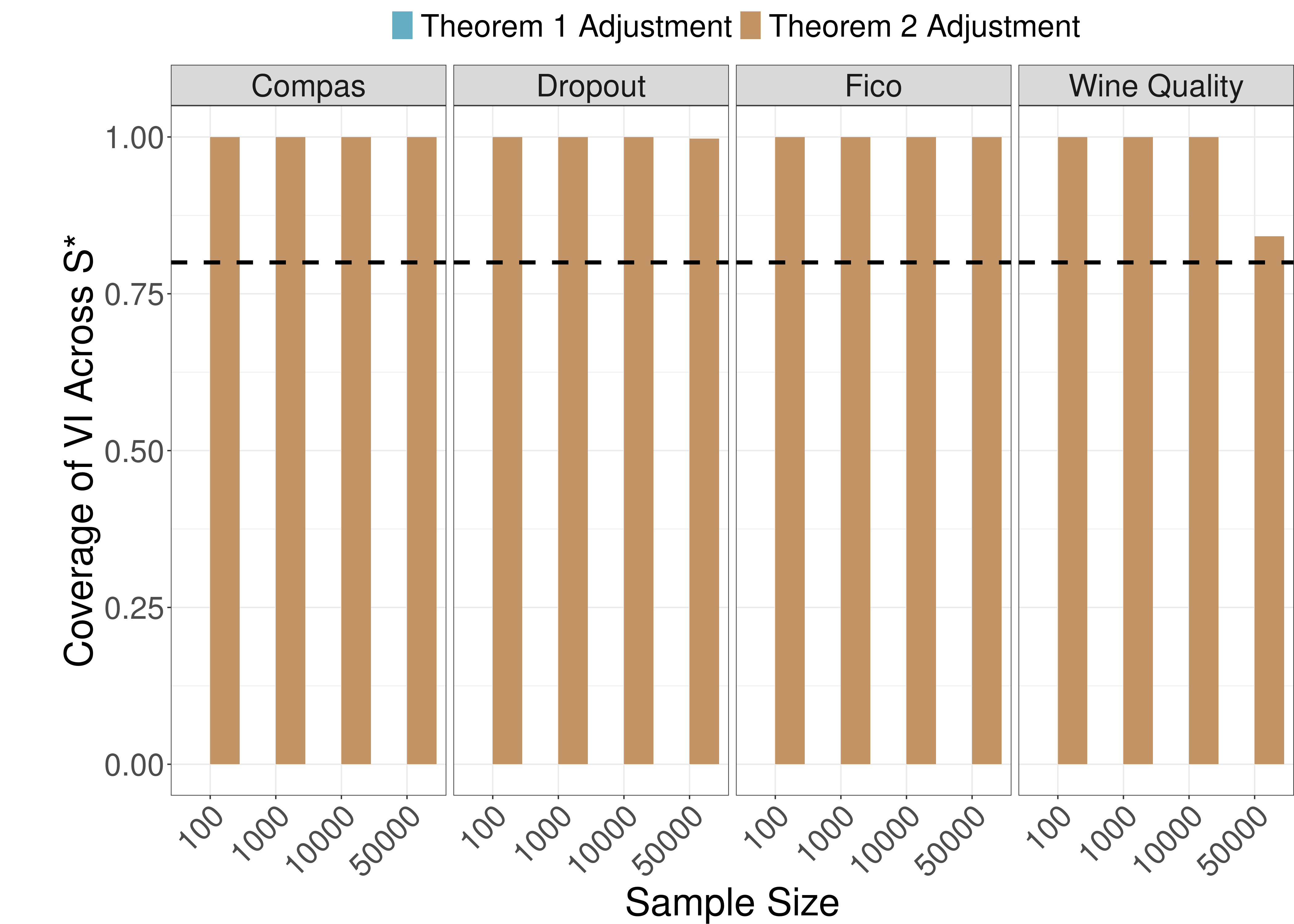}
    \caption{Verifying Theorem~\ref{thm:vi_coverage_uncond} using random forests.
    We achieve the specified coverage rate of $\geq0.8$ only when adjusting for (i) model uncertainty via Theorem~\ref{thm:smart_eps} and (ii) variable importance estimation uncertainty (in gold). Adjusting for model uncertainty alone (in blue; here, always 0) is not sufficient. For each setting, we compute the proportion of 100 experiments where our variable importance bounds capture the true variable importance for all submodels $f_u \in \allcondmodels$, averaged over variables.}
    \label{fig:rf_vi_coverage_uncond}
\end{figure*}

\subsection{Recovering variable importance for each model in $\allcondmodels$} We now turn to evaluate Theorem~\ref{thm:vi_coverage_uncond} for the random forest Rashomon set. Given a finite sample bound for the estimation of our variable importance metric for a specified model, we can cover the importance of each variable to each model in $\allcondmodels$ with at least a specified probability. In this experiment, we set our target probability to $1-(\delta + \gamma) = 0.8$, and apply the finite sample bound for subtractive model reliance (MR) from Equation B.26 of \citet{fisher2019all}.
Figure \ref{fig:rf_vi_coverage_uncond} demonstrates that adjusting only for model uncertainty as in Corollary~\ref{thm:smart_eps} yields invalid bounds on model-level variable importance. In fact, omitting variable importance uncertainty quantification yields intervals that \textit{never} contain the true variable importance for all conditional submodels. In contrast, 
\textbf{across all four datasets and at all sample sizes, our approach achieves nominal coverage.}

\begin{figure*}
    \centering
    \includegraphics[width=0.7\linewidth]{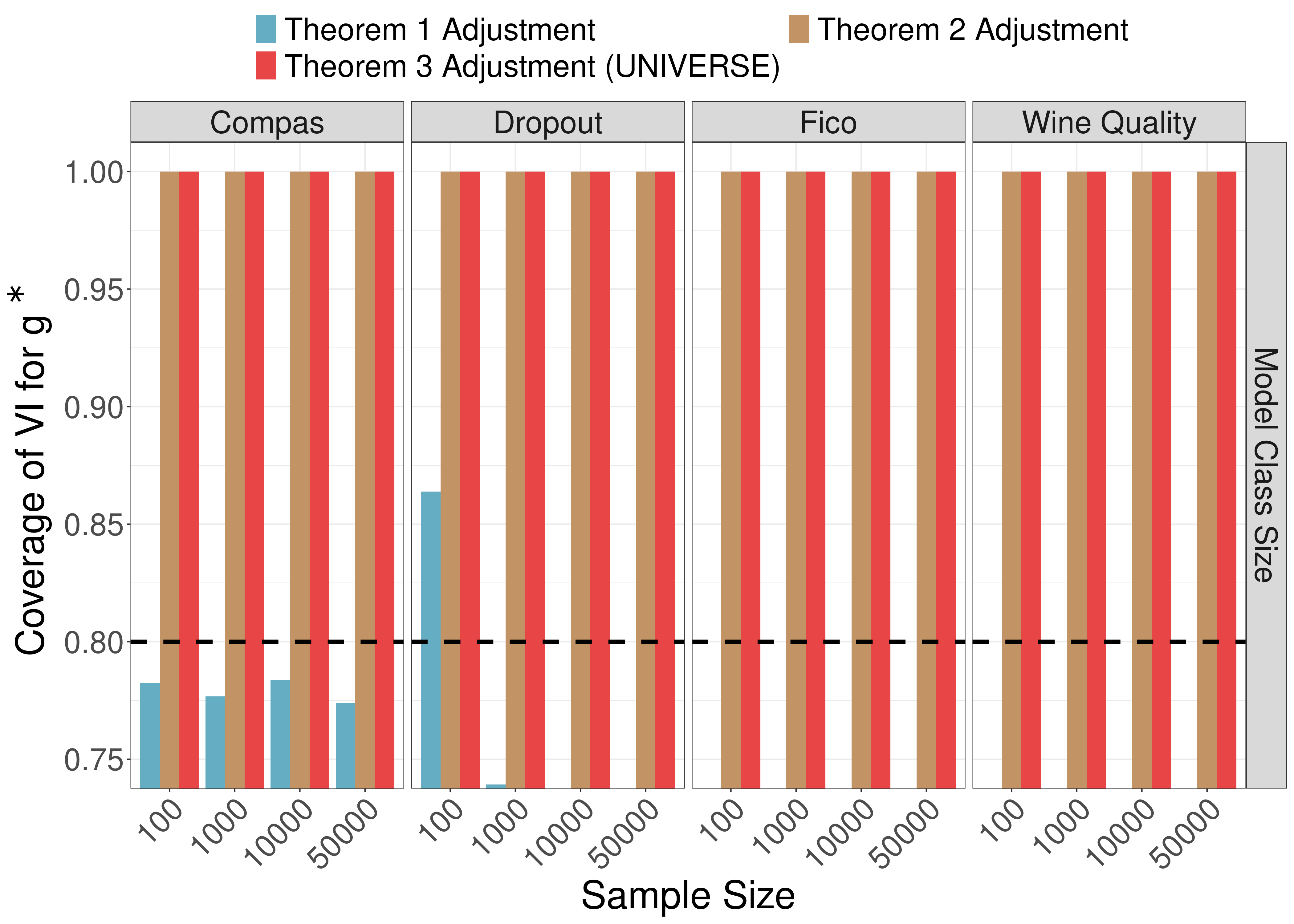}
    \caption{Verifying Theorem~\ref{thm:optmodel_coverage} using random forests.
    We consistently achieve the specified coverage rate of $\geq 0.8$ when we account for (i) model uncertainty, (ii) variable importance uncertainty, and (iii) VI drift. 
    Each bar measures the proportion of 100 experiments in which our bounds capture the true variable importance for the true model $\optmodel$. Plots are colored such that blue only accounts for finite sample model uncertainty as in Theorem~\ref{thm:smart_eps}, gold also adjusts for uncertainty in estimating subtractive model reliance (MR) at the model-level as in Theorem~\ref{thm:vi_coverage_uncond}, and red adjusts for the previous two \textit{and} distribution shifts induced by omitted variables \ref{thm:optmodel_coverage}. Applying al three adjustments yields the target coverage rate of $\geq 0.8$, with $\delta=\gamma=0.1$.}
    \label{fig:rf_dgp_coverage}
\end{figure*}

\subsection{Recovering variable importance for the $\optmodel$ with unobserved features} Finally, we evaluate Theorem \ref{thm:optmodel_coverage} for the random forest Rashomon set, which provides high probability bounds on variable importance to the true model $\optmodel$, even with unobserved variables. We again set our target coverage rate to $1-(\delta + \gamma) = 0.8$ and apply the finite sample bound for subtractive MR \citet{fisher2019all}. This reflects applying the entire \ourmethodAcronym{} framework.

Figure \ref{fig:rf_dgp_coverage} demonstrates that \textbf{we consistently cover the true importance to $\optmodel$ in more than the specified $80\%$ of cases across all four datasets}, even though $\optmodel$ depends on unobserved variables. 

In contrast to our experiments with the Rashomon set of sparse decision trees in Figure \ref{fig:dgp_coverage}, we find that only adjusting for uncertainty in variable importance (Theorem 2 Adjustment) tends to yield the nominal coverage rate. This is because the model class of random forests is much more expressive than that of sparse decision trees, which results in more varied Rashomon sets and ultimately looser bounds on variable importance. This makes sense when the chosen model class is viewed as a prior; if a practitioner believes a smaller, less expressive model class is sufficient to model capture the data generation process, this reflects a stronger prior than using a large, expressive model class, and results in tighter bounds.

\clearpage
\section{EXPERIMENTS WITH MODEL CLASS SIZE} \label{app-sec:experiments_w_model_class_size}
Section~\ref{sec:experiments} evaluates each of our primary theoretical claims empirically by first estimating the size of the Rashomon set. Because we are using an \textit{estimate} of the Rashomon set size, these bounds are not necessarily guaranteed to be finite-sample valid. For scientists worried about finite-sample validity, we present an alternative strategy where we use the size of the model class as our upper bound on the size of $\allcondmodels.$

\begin{figure}
    \centering
    \includegraphics[width=1\linewidth]{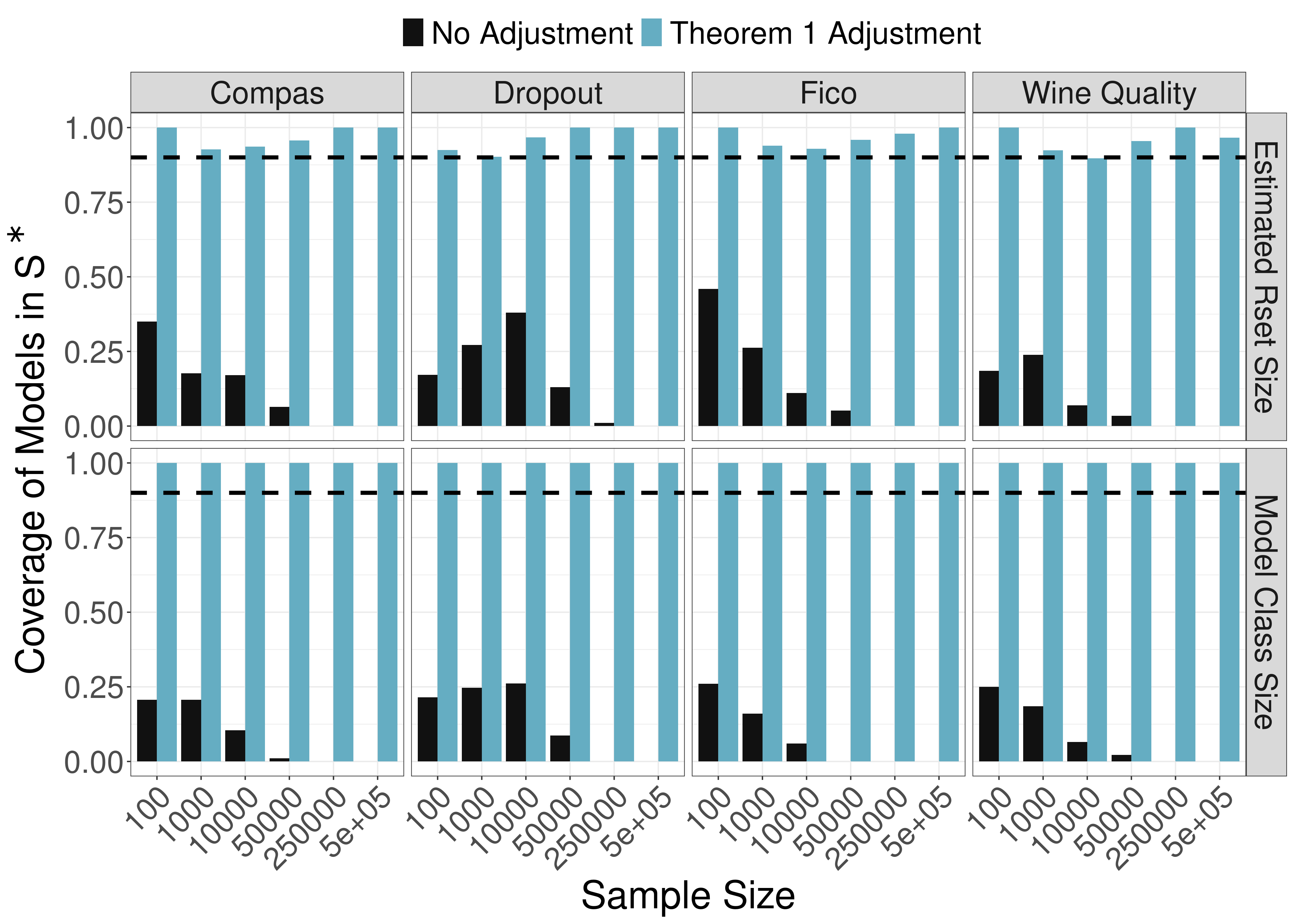}
    \caption{Verifying Theorem~\ref{thm:smart_eps} in finite sample datasets. We compute the proportion of 100 random draws of the each dataset in which Rashomon sets estimated with the Rashomon threshold adjusting for finite sample biases as in Theorem \ref{thm:smart_eps} (in blue) and without any adjustment (in black) captures each $\condsubmodelNoInput{u}$ for each setting. The target coverage rate is $\geq 0.9$, with $\delta=0.1$. Across all sample sizes and datasets, omitting finite sample adjustments yields Rashomon sets that leave out necessary models. In contrast, our adjustment yields the target coverage rate, verifying the theorem holds.  
    We use the estimated Rashomon set size (top row) and the model class size (bottom row) as our upper bounds on the size of $\allcondmodels$.}
    \label{fig:app-thm1}
\end{figure}

Figure~\ref{fig:app-thm1} displays the results verifying Theorem~\ref{thm:smart_eps}. The top row displays the main paper results again while the bottom row displays the new results; in the bottom row, not adjusting for finite sample or regularization behavior leads to overly conservative Rashomon sets with severe undercoverage of the models in $\allcondmodels$. In contrast, our approach with an adjusted Rashomon threshold guarantees coverage at all sample sizes. However, this conservative correction is overly conservative, leading to Rashomon sets that contains $\allcondmodels$ at \textit{all} sample sizes. In contrast, using an estimate of the Rashomon set size leads to Rashomon sets that are probabilistically more valid.

\begin{figure}
    \centering
    \includegraphics[width=1\linewidth]{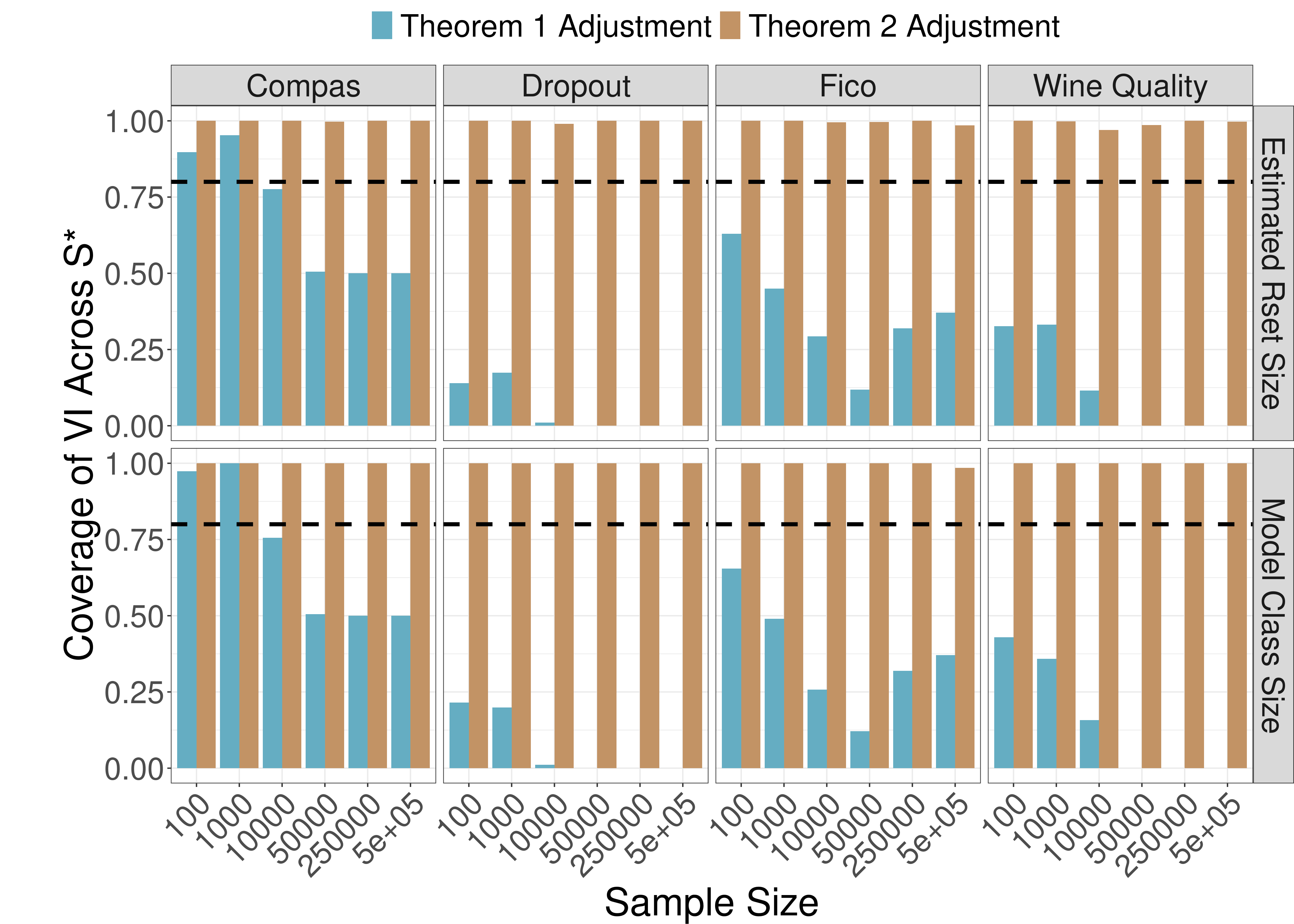}
    \caption{Verifying Theorem~\ref{thm:vi_coverage_uncond}.
    We achieve the specified coverage rate of $\geq0.8$ only when adjusting for (i) model uncertainty via Theorem~\ref{thm:smart_eps} and (ii) variable importance estimation uncertainty (in gold). Adjusting for model uncertainty alone (in blue) is not sufficient. For each setting, we compute the proportion of 100 experiments where our variable importance bounds capture the true variable importance for all submodels $f_u \in \allcondmodels$, averaged over variables. We use the estimated Rashomon set size (top row) and the model class size (bottom row) as our upper bounds on the size of $\allcondmodels$.}
    \label{fig:app-thm2}
\end{figure}

Figure~\ref{fig:app-thm2} displays the results verifying Theorem~\ref{thm:vi_coverage_uncond}. The top row displays the main paper results again while the bottom row displays the new results; in the bottom row, adjusting only for model uncertainty and not VI-estimation uncertainty yields bounds that undercover the variable importance for models in $\allcondmodels$. In contrast, our approach that accounts for VI estimation uncertainty yields bounds that achieve the specified error rate. However, this conservative correction is overly conservative, leading to bounds that contain the true VI for all models in $\allcondmodels$ always. In contrast, using an estimate of the Rashomon set size leads to bounds that are slightly less conservative; for example, the Wine Quality dataset at 10,000 and 50,000 samples creates bounds that have coverage $<100\%.$

\begin{figure}
    \centering
    \includegraphics[width=1\linewidth]{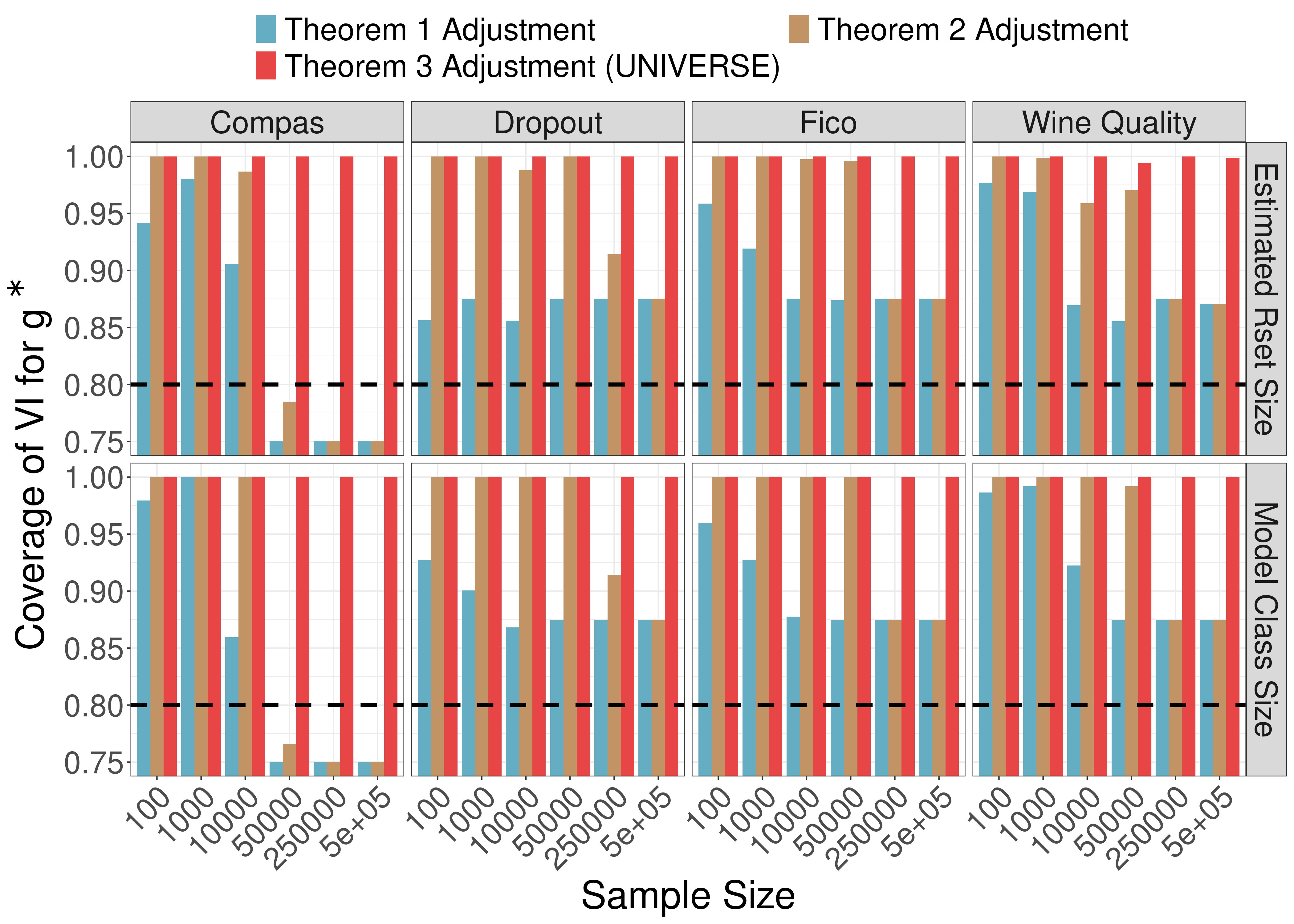}
    \caption{Verifying Theorem~\ref{thm:optmodel_coverage}.
    We consistently achieve the specified coverage rate of $\geq 0.8$ only when we account for (i) model uncertainty, (ii) variable importance uncertainty, and (iii) VI drift. 
    Each bar measures the proportion of 100 experiments in which our bounds capture the true variable importance for the true model $\optmodel$. Plots are colored such that blue only accounts for finite sample model uncertainty as in Theorem~\ref{thm:smart_eps}, gold also adjusts for uncertainty in estimating subtractive model reliance (MR) at the model-level as in Theorem~\ref{thm:vi_coverage_uncond}, and red adjusts for the previous two \textit{and} distribution shifts induced by omitted variables \ref{thm:optmodel_coverage}. All three adjustments are necessary to achieve the target coverage rate of $\geq 0.8$, with $\delta=\gamma=0.1$.  We use the estimated Rashomon set size (top row) and the model class size (bottom row) as our upper bounds on the size of $\allcondmodels$.}
    \label{fig:app-thm3}
\end{figure}

Finally, Figure~\ref{fig:app-thm3} displays the results verifying Theorem~\ref{thm:optmodel_coverage}. The top row displays the main paper results again while the bottom row displays the new results; in the bottom row, not adjusting for VI drift yields bounds that undercover the variable importance for the true model $g^*.$ For example, the blue and gold bars that represent only performing adjustments from Theorem~\ref{thm:smart_eps} and Theorem~\ref{thm:vi_coverage_uncond} respectively achieve a coverage of only 0.76 on the Compas dataset even at extremely large sample sizes of 500,000 even when using the conservative model class size adjustment (bottom row). In contrast, our approach remains valid. However, using the conservative model class size is overly conservative with bounds that contain the true VI at all sample sizes for all datasets. In contrast, using the estimated Rashomon set size yields slightly tighter intervals that can achieve $<100\%$ coverage (e.g., Wine Quality dataset with sample size 50,000).

\section{INTERVAL WIDTHS} \label{app-sec:experiments_w_model_class_size}
In this section, we evaluate the width of our intervals. We display the average interval width for experiment conducted to evaluate coverage in Section~\ref{sec:experiments}. As the sample size increases, our intervals become much tighter--regardless of whether we estimate the Rashomon set size or use the model class size as our upper bound on the size of $\allcondmodels$. Importantly, each additional adjustment increases interval widths very slightly but contribute to coverage at the specified rate. Our adjustments are therefore a small cost to pay for improved Type-1 error control, especially at larger sample sizes.

\begin{figure}
    \centering
    \includegraphics[width=1\linewidth]{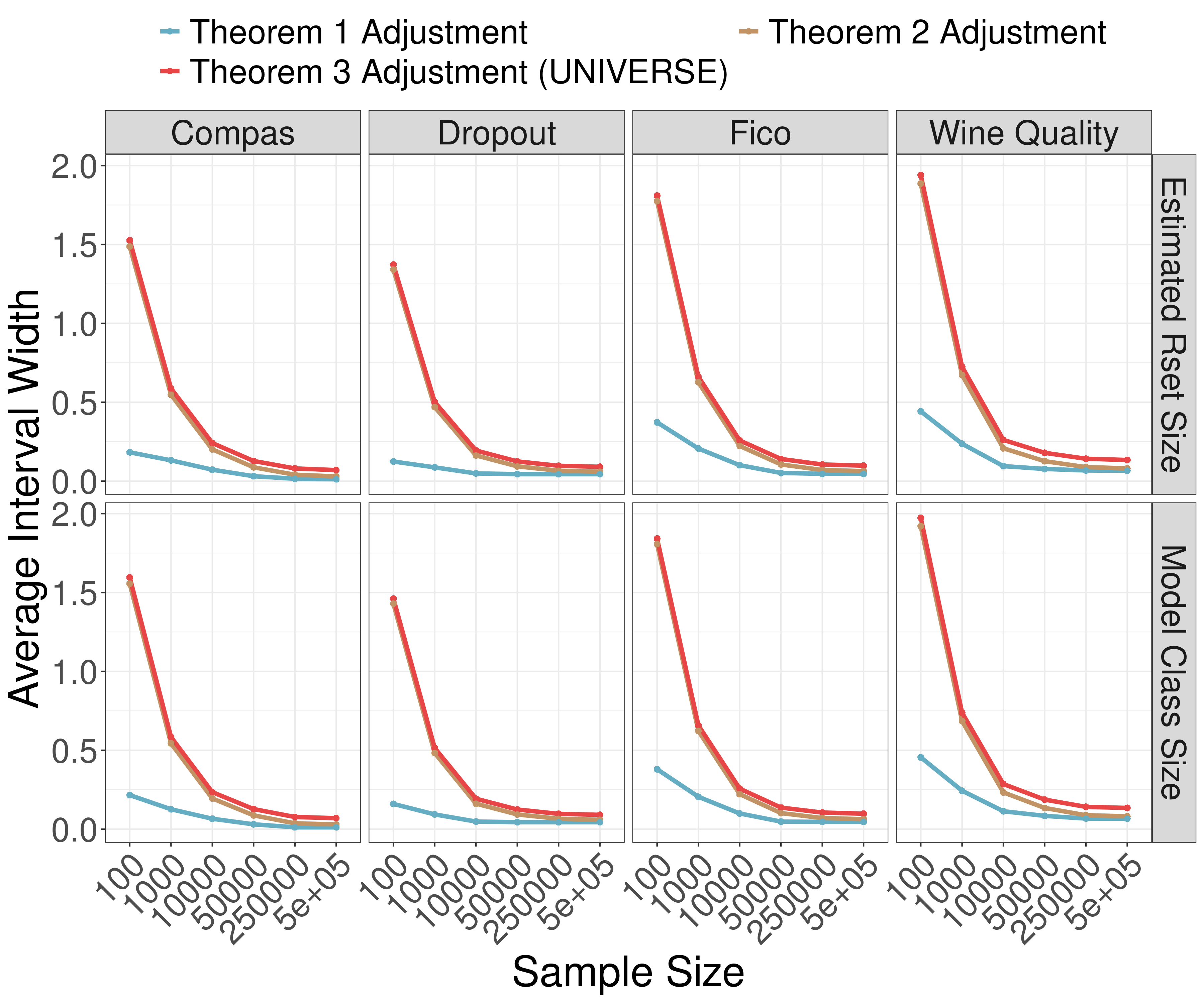}
    \caption{Evaluating the width of our intervals. We display the average interval width for each experiment conducted to evaluate coverage in Section~\ref{sec:experiments}. As the sample size increases, our intervals become much tighter--regardless of whether we estimate the Rashomon set size or use the model class size as our upper bound on the size of $\allcondmodels$.}
    \label{fig:placeholder}
\end{figure}

\end{document}